\documentclass[11pt]{article}

\usepackage[letterpaper, top=1in, bottom=1in, left=1.3in, right=1.3in]{geometry}

\usepackage{algorithmic}
\usepackage[linesnumbered,ruled,vlined]{algorithm2e}
\SetAlgoInsideSkip{0pt}

\usepackage[utf8]{inputenc} % allow utf-8 input
\usepackage[T1]{fontenc}    % use 8-bit T1 fonts
\usepackage{hyperref}       % hyperlinks
\usepackage{url}            % simple URL typesetting
\usepackage{booktabs}       % professional-quality tables
\usepackage{amsfonts}       % blackboard math symbols
\usepackage{nicefrac}       % compact symbols for 1/2, etc.
\usepackage{microtype}      % microtypography
\usepackage{xcolor}         % colors

\usepackage{amsmath,amssymb,amsthm,amsfonts,mathrsfs}
\usepackage{mathrsfs}
\usepackage{mathtools}
\usepackage{dsfont}
\usepackage{nicefrac}

\newcommand*{\diff}{\mathop{}\!\mathrm{d}}

\usepackage{graphicx}
\usepackage{subfigure}
\usepackage{wrapfig}

\theoremstyle{definition}
\newtheorem{assumption}{Assumption}
\newtheorem{theorem}{Theorem}
\newtheorem{definition}{Definition}
\newtheorem{corollary}{Corollary}
\newtheorem{lemma}{Lemma}
\newtheorem{remark}{Remark}

\newtheorem{condition}{Condition}

\title{Second-Order Convergence in Private Stochastic Non-Convex Optimization}

\author{
Youming Tao\\
TU Berlin\\
\texttt{tao@ccs-labs.org}
\and
Zuyuan Zhang\\
George Washington University\\
\texttt{zuyuan.zhang@gwu.edu}
\and
Dongxiao Yu\\
Shandong University\\
\texttt{dxyu@sdu.edu.cn}
\and
Xiuzhen Cheng\\
Shandong University\\
\texttt{xzcheng@sdu.edu.cn}
\and
Falko Dressler\\
TU Berlin\\
\texttt{dressler@ccs-labs.org}
\and
Di Wang\\
KAUST\\
\texttt{di.wang@kaust.edu.sa}
}
\date{}

\begin{document}

\maketitle

\begin{abstract}
  We investigate the problem of finding second-order stationary points (SOSP) in differentially private (DP) stochastic non-convex optimization. Existing methods suffer from two key limitations: \textbf{(i)} inaccurate convergence error rate due to overlooking gradient variance in the saddle point escape analysis, and \textbf{(ii)} dependence on auxiliary private model selection procedures for identifying DP-SOSP, which can significantly impair utility, particularly in distributed settings. To address these issues, we propose a generic perturbed stochastic gradient descent (PSGD) framework built upon Gaussian noise injection and general gradient oracles. A core innovation of our framework is using model drift distance to determine whether PSGD escapes saddle points, ensuring convergence to approximate local minima without relying on second-order information or additional DP-SOSP identification. By leveraging the adaptive DP-SPIDER estimator as a specific gradient oracle, we develop a new DP algorithm that rectifies the convergence error rates reported in prior work. We further extend this algorithm to distributed learning with heterogeneous data, providing the first formal guarantees for finding DP-SOSP in such settings. Our analysis also highlights the detrimental impacts of private selection procedures in distributed learning under high-dimensional models, underscoring the practical benefits of our design. Numerical experiments on real-world datasets validate the efficacy of our approach.
\end{abstract}

\section{Introduction}

Stochastic optimization is a fundamental problem in machine learning and statistics, aimed at training models that generalize well to unseen data using a finite sample drawn from an unknown distribution. As the volume of sensitive data continues to grow, privacy has become a pressing concern. This has led to the widespread adoption of differential privacy (DP)~\cite{dwork2006calibrating}, which provides rigorous privacy guarantees while preserving model utility in learning tasks.

In the past decade, significant progress has been made in DP stochastic optimization, particularly for convex objectives~\cite{choquette2024optimal, liu2024user, su2023differentially, su2022faster, tao2022private,tao2024private}. While convex problems are relatively well understood, non-convex optimization introduces unique challenges, primarily due to the presence of saddle points. Most existing DP algorithms for non-convex problems focus on finding first-order stationary points (FOSP), characterized by small gradient norms~\cite{arora2023faster, bassily2021differentially, zhou2020private,zhang2025improved}. However, FOSP include not only local minima but also saddle points and local maxima, often leading to suboptimal solutions~\cite{jain2015computing, sun2016geometric}. Consequently, second-order stationary points (SOSP), where the gradient is small and the Hessian is positive semi-definite, are more desirable as they guarantee convergence to local minima.

Motivated by this, substantial research has been devoted to finding SOSP in non-convex optimization~\cite{fang2019sharp, jin2021nonconvex, daneshmand2018escaping, jin2017escape, ge2015escaping}. However, the study of SOSP under differential privacy constraints (DP-SOSP) remains limited. At the same time, distributed learning has become increasingly important for training large-scale models across decentralized edge devices. Yet, no existing work has addressed DP-SOSP in non-convex stochastic optimization under distributed settings. Compared to single-machine setups, distributed learning introduces additional challenges, including data heterogeneity, cross-participant privacy, and communication efficiency.

\noindent{\bf Limitations in the State-of-the-Art.} A notable exception in the study of DP-SOSP for stochastic optimization is the recent work by~\cite{liu2024private}, which injects additional Gaussian noise into the DP gradient estimator near saddle points to facilitate escape. Despite its contributions, this method suffers from two key limitations. \textbf{(i)} Its saddle point escape analysis overlooks the variance of gradients, leading to incorrect error bounds. A direct correction of the analysis would unfortunately yield a weaker type of SOSP guarantee than originally targeted. This is because their design relies on additional injected noise beyond the inherent DP noise for escape, highlighting the need for an effective way of exploiting the  DP noise already present. \textbf{(ii)} Their learning algorithm outputs all model iterates and guarantees only the \emph{existence} of a DP-SOSP, requiring an auxiliary private model selection procedure to identify one. While effective in single-machine settings, it faces critical issues in distributed environments due to decentralized data access. In particular, auxiliary private selection introduces non-negligible error and communication overhead, especially when sharing high-dimensional second-order information. These drawbacks also underscore the necessity of a new learning algorithm that inherently outputs a DP-SOSP without dependence on any additional private selection procedure. 

\noindent{\bf Our Contributions.} 
We refer the reader to Appendix~\ref{sec-gap} for more detailed discussions of the limitations outlined above. To address the challenges identified above, we propose a generic algorithmic and analytical framework for finding DP-SOSP in stochastic non-convex optimization. Our approach not only corrects existing error rates but also extends naturally to distributed learning. The main contributions are summarized as follows:

\noindent{\bf 1. A generic non-convex stochastic optimization framework:}  
We introduce a perturbed stochastic gradient descent (PSGD) framework that employs Gaussian noise and general stochastic gradient oracles. This framework serves as a versatile optimization tool for non-convex stochastic problems beyond the DP setting. A key innovation is a novel criterion based on model drift distance, which enables provable saddle point escape and guarantees convergence to approximate local minima with low iteration complexity and high probability.

\noindent{\bf 2. Corrected error rates for DP non-convex optimization:} By incorporating the adaptive DP-SPIDER estimator as the gradient oracle, we develop a differentially private algorithm that achieves a corrected error rate bound of $\tilde{O}\big(\frac{1}{n^{1/3}} + \big(\frac{\sqrt{d}}{\epsilon n}\big)^{2/5}\big)$, where $n$ is the number of samples. This corrects the suboptimal bound of $\tilde{O}\big(\frac{1}{n^{1/3}} + \big(\frac{\sqrt{d}}{\epsilon n}\big)^{3/7}\big)$ reported in~\cite{liu2024private}.

\noindent{\bf 3. Application to distributed learning:}  
We extend the adaptive DP-SPIDER estimator to distributed learning. Via adaptivity, our learning algorithm improves upon the DIFF2~\cite{murata2023diff2}, which only guarantees convergence to DP-FOSP under \textit{homogeneous} data. In contrast, our method provides the first error bound for converging to DP-SOSP under \textit{heterogeneous} data: $\tilde{O}\big(\frac{1}{(mn)^{1/3}} + \big(\frac{\sqrt{d}}{\epsilon mn}\big)^{2/5}\big)$, where $m$ is the number of participants and $n$ is the number of samples per participant. Furthermore, we analyze the adverse effects of private model selection, showing that it deteriorates utility guarantees in high-dimensional regimes, thereby highlighting the necessity of our framework.

Due to the space limit, \textbf{literature review}, \textbf{technical lemmata}, \textbf{further discussions}, \textbf{omitted proofs}, \textbf{experimental results}, \textbf{broader impacts} and \textbf{conclusions} are all included in the Appendix.
\section{Related Work}

\paragraph{Private Stochastic Optimization} 

Differential privacy (DP) has become a crucial consideration in stochastic optimization due to increasing concerns about data privacy. The pioneering work by \cite{dwork2006calibrating} established the foundational principles of DP, and its application in stochastic optimization has since seen significant progress. Early efforts primarily focused on convex optimization,  achieving strong privacy guarantees while ensuring efficient learning, with a long list of representative works e.g., \cite{bassily2014private,wang2017differentially,wang2018empirical,bassily2019private,wang2020empirical,wang2020differentially,feldman2020private,bassily2021differentially,hu2022high,tao2022private,su2023differentially,choquette2024optimal,su2024faster}. Recent advances have extended DP to non-convex settings, mainly focusing on first-order stationary points (FOSP). Notable works in this area include \cite{wang2019differentially,zhou2020private,bassily2021differentially,xiao2023theory,arora2023faster}, which improved error rates in non-convex optimization with balanced privacy and utility in stochastic gradient methods. However, these works generally fail to address the more stringent criterion of second-order stationary points (SOSP). The very recent work \cite{liu2024private} tired to narrow this gap, but unfortunately has some issues in their results as we discussed before. Our work builds on this foundation by correcting error rates and proposing a framework that ensures convergence to SOSP while maintaining DP.

\paragraph{Finding Second-Order Stationary Points (SOSP)}

In non-convex optimization, convergence to FOSP is often insufficient, as saddle points can lead to sub-optimal solutions \cite{jain2015computing,sun2016geometric}. Achieving SOSP, where the gradient is small and the Hessian is positive semi-definite, ensures that the optimization converges to a local minimum rather than a saddle point. Techniques for escaping saddle points, such as perturbed SGD with Gaussian noise, have been explored in works like \cite{ge2015escaping} and \cite{jin2021nonconvex}. \cite{ge2015escaping} first showed that SGD with a simple parameter perturbation can escape saddle points efficiently. Later, the analysis was refined by \cite{jin2017escape,jin2021nonconvex}. Recently, variance reduction techniques have been applied to second-order guaranteed methods \cite{ge2019stabilized,li2019ssrgd}.These methods ensure escape from saddle points by introducing noise to the gradient descent process. In contrast, the studies of SOSP under DP are quite limited, and most of them only consider the empirical risk minimization objective, such as \cite{wang2019differentially,wang2021escaping,avdiukhin2024noise}. Very recently, \cite{liu2024private} addressed the population risk minimization objective, but with notable gaps in their error analysis, particularly in the treatment of gradient variance. Moreover, all of these works are limited to the single-machine setting and cannot be directly extended to the more general distributed learning setting. 

\paragraph{Distributed Learning}

With the rise of large-scale models and decentralized data, distributed learning has gained significant attention. Methods like federated learning \cite{mcmahan2017communication} have enabled multiple clients to collaboratively train models without sharing their local data. Recent studies, such as \cite{gao2024private,xiang2023practical,lowy2023privatea,lowy2023privateb} have investigated DP learning in distributed settings, but these works are limited to first-order optimality. While some studies have investigated SOSP in distributed learning, their focus was primarily on Byzantine-fault tolerance~\cite{yin2019defending}, and communication efficiency~\cite{murata2022escaping,chen2024escaping}. No effort, to our knowledge, has been made to ensure DP-SOSP in distributed learning scenarios with heterogeneous data. Our proposed framework fills this gap by introducing the first distributed learning algorithm with DP-SOSP guarantees while effectively handling data heterogeneity across clients.

\section{Limitations of the State-of-the-Art}\label{sec-gap}

\subsection{Limitation~1: Flawed Error Rate Analysis}

\paragraph{Gradient variance overlooked in saddle point escape.} 

The error rate bound for finding a DP-SOSP in~\cite{liu2024private} is fundamentally incorrect. Their analysis relies on Lemma 3.4 therein (adapted from~\cite[Lemma 12]{wang2019differentially}), which claims that adding Gaussian noise at the same scale as the DP gradient estimation error suffices to reduce the function value with high probability, enabling escape from saddle points. This argument critically depends on proving that the region around a saddle point where SGD may get stuck is sufficiently narrow. Under this condition, perturbation along the escape direction ensures that the SGD sequence can escape with high probability.

However, the analysis neglects a key factor, which is the stochastic gradient variance. Their proof implicitly uses exact gradients of the population risk, which are unavailable to the algorithm. This is evidenced by the equation preceding equation (39) in~\cite{wang2019differentially}. Another indication of this oversight is their choice of step size $\eta = 1/M$. While valid for gradient descent with exact gradients, prior work~\cite{jin2021nonconvex} has shown that stochastic gradients require a smaller step size. The use of $\eta = 1/M$ in~\cite{liu2024private} for population risk minimization reflects a failure to account for gradient stochasticity. This leads to an underestimated gradient complexity and an overestimated effective sample size per gradient estimate, which ultimately results in an overly optimistic error rate. A correct analysis must acknowledge that stochastic gradients increase estimation error, implying that the true error rate for finding a DP-SOSP is weaker than the one reported.

\paragraph{Fixing the proof is insufficient, a new algorithm is necessary.} 

Although the analytical error can be identified, correcting the proof alone does not yield a satisfactory result. Any direct correction would only achieve a weaker $(\alpha, \alpha^{2/5})$-SOSP guarantee, rather than the desired $\alpha$-SOSP. In particular, the second-order accuracy would degrade to $\widetilde{O}(\alpha^{2/5})$ instead of the ideal $\widetilde{O}(\alpha^{1/2})$.

This limitation arises because the algorithm in~\cite{liu2024private} can be viewed as a special case of perturbed gradient descent with bounded gradient inexactness as developed in~\cite{yin2019defending}, where the DP noise contributes to the perturbation. By invoking~\cite[Theorem 3]{yin2019defending}, one only obtains an error rate bound with respect to a weaker class of SOSP where the second-order accuracy depends on $\widetilde{O}(\alpha^{2/5})$.

The underlying reason is that both~\cite{yin2019defending} and~\cite{liu2021flame} rely on injecting additional noise to facilitate escape from saddle points, without considering the role of inherent DP Gaussian noise in the gradients. The excessive injected noise degrades the SOSP guarantee.

To fully resolve this issue, a new algorithmic design is required. In the setting of~\cite{yin2019defending}, where gradient perturbations stem from adversarial attacks, such degradation is unavoidable since the perturbations can hinder rather than assist escape. However, in the DP setting, the Gaussian noise is well-behaved and can naturally aid saddle point escape. By leveraging the inherent DP noise, it becomes possible to avoid the need for additional injected noise and to achieve $\alpha$-SOSP convergence as desired. Therefore, relying on the algorithmic designs of~\cite{yin2019defending} or~\cite{liu2021flame} is insufficient, and a new algorithm must be developed to achieve the desired guarantees.

\subsection{Limitation~2: Challenges of Private SOSP Selection}

\paragraph{Inapplicability of AboveThreshold in distributed learning.} 

The algorithm in~\cite{liu2024private} guarantees only the existence of an $\alpha$-SOSP among its iterates. To privately identify such a point, it applies the AboveThreshold mechanism to test whether candidate models satisfy the SOSP conditions by privately evaluating gradient norms and Hessian eigenvalues. While this procedure introduces negligible error in single-machine settings, it faces fundamental challenges in distributed learning.

According to~\cite[Lemma 4.5]{liu2024private}, for any $x \in \mathbb{R}^d$ and a dataset $S$ of size $O(n)$, with probability at least $1 - \omega$, the following holds:
\[
\|\nabla F_{\mathcal{D}}(x) - \nabla \hat{f}_S(x)\| \le O\left( \frac{G \log (d/\omega)}{\sqrt{n}} \right), \quad
\|\nabla^2 F_{\mathcal{D}}(x) - \nabla^2 \hat{f}_S(x)\|_{\mathrm{op}} \le O\left( \frac{M \log (d/\omega)}{\sqrt{n}} \right).
\]
This implies:
\[
\|\nabla \hat{f}_S(x)\| \le \|\nabla F_{\mathcal{D}}(x)\| + O\left( \frac{G \log \frac{d}{\omega}}{\sqrt{n}} \right), 
\lambda_{\min} (\nabla^2 \hat{f}_S(x)) \ge \lambda_{\min} (\nabla^2 F_{\mathcal{D}}(x)) - O\left( \frac{M \log \frac{d}{\omega}}{\sqrt{n}} \right).
\]

With these bounds, AboveThreshold can identify a DP-SOSP by setting appropriate thresholds. However, this procedure relies on centralized access to the dataset $S$.

In distributed learning, each client holds a local dataset $S_i$. To estimate global quantities, aggregation is required:
\[
\|\nabla \hat{f}_S(x)\| \le \frac{1}{m} \sum_{i=1}^{m} \|\nabla \hat{f}_{S_i}(x)\|, \quad
\lambda_{\min} (\nabla^2 \hat{f}_S(x)) \ge \frac{1}{m} \sum_{i=1}^{m} \lambda_{\min} (\nabla^2 \hat{f}_{S_i}(x)).
\]
Yet the learning algorithm guarantees only:
\[
\|\nabla F_{\mathcal{D}}(x)\| \le \frac{1}{m} \sum_{i=1}^{m} \|\nabla F_{\mathcal{D}_i}(x)\|, \quad
\lambda_{\min} (\nabla^2 F_{\mathcal{D}}(x)) \ge \frac{1}{m} \sum_{i=1}^{m} \lambda_{\min} (\nabla^2 F_{\mathcal{D}_i}(x)),
\]
This relationship does not provide an upper bound on $\|\nabla \hat{f}_S(x)\|$ or a lower bound on $\lambda_{\min} (\nabla^2 \hat{f}_S(x))$ solely from local empirical estimates. Therefore, it is infeasible to determine valid thresholds for AboveThreshold based only on local information. Any attempt to perform this selection would require clients to share their (noisy) gradients and Hessians with the server, which introduces substantial privacy, communication, and computation costs.

\paragraph{Eliminating private model selection is essential in distributed learning.} 

A feasible method for private model selection in distributed learning would extend the centralized algortihm of~\cite[Algorithm 5]{wang2019differentially}. Specifically, each client privately computes gradients and Hessians on additional local data beyond the training set, and the server aggregates these to estimate global quantities. However, this strategy has several drawbacks. It requires extra data outside the training process, increases communication overhead by transmitting high-dimensional gradients and Hessians, and incurs high computational costs. It also shifts the method from a first-order to a second-order algorithm.

Moreover, as shown in Section~\ref{sec-distributed}, sharing perturbed high-dimensional gradients and Hessians, rather than one-dimensional scalar queries as in AboveThreshold, introduces non-negligible additional error. This error accumulation degrades the accuracy guarantees provided by the learning algorithm. Unlike the single-machine case, private model selection in distributed learning incurs significant costs in accuracy, privacy, computation, and communication.

These challenges demonstrate the necessity of designing an algorithm that inherently outputs a DP-SOSP without relying on a private model selection procedure. Such a design avoids additional data consumption, computational burden, communication overhead, and deterioration of error guarantees.

\section{Preliminaries}\label{sec-pre}

\noindent{\bf Notations.} We denote by $\|\cdot\|$ the $\ell_2$ norm and by $\lambda_{\min}(\cdot)$ the smallest eigenvalue of a matrix. The symbol $\mathbf{I}_d$ represents the $d$-dimensional identity matrix. We use $O(\cdot)$ and $\Omega(\cdot)$ to hide constants independent of problem parameters, while $\tilde{O}(\cdot)$ and $\tilde{\Omega}(\cdot)$ additionally hide polylogarithmic factors.

\noindent{\bf Stochastic Optimization.} Let $f : \mathbb{R}^d \times \mathcal{Z} \to \mathbb{R}$ be a (potentially non-convex) loss function, where $x \in \mathbb{R}^d$ denotes the $d$-dimensional model parameter and $z \in \mathcal{Z}$ is a data point.

\begin{assumption}\label{assump-loss}
    The loss function $f(\cdot; z)$ is $G$-Lipschitz, $M$-smooth, and $\rho$-Hessian Lipschitz. Specifically, for any $z \in \mathcal{Z}$ and any $x_1, x_2 \in \mathbb{R}^d$, we have:  
    (i) $|f(x_1; z) - f(x_2; z)| \le G \|x_1 - x_2\|$;  
    (ii) $\|\nabla f(x_1; z) - \nabla f(x_2; z)\| \le M \|x_1 - x_2\|$;  
    (iii) $\|\nabla^2 f(x_1; z) - \nabla^2 f(x_2; z)\| \le \rho \|x_1 - x_2\|$.
\end{assumption}

Let $\mathcal{D}$ denote the unknown data distribution. The population risk is defined as the \textit{expected} loss: $F_\mathcal{D}(x) \coloneqq \mathbb{E}_{z \sim \mathcal{D}} [f(x; z)]$ for $\forall x \in \mathbb{R}^d$. When clear from context, we omit $\mathcal{D}$ and simply write $F(x)$.

\begin{assumption}\label{assump-func}
Let $x^*$ denote a minimizer of the population risk and $F^* = F(x^*)$ its minimum value. There exists $U \in \mathbb{R}$ such that $\max_x F(x) - F^* \le U$.
\end{assumption}

Let $D$ denote a dataset of $n$ i.i.d. samples from $\mathcal{D}$. The empirical risk is defined as $\hat{f}_D(x) \coloneqq \frac{1}{|D|} \sum_{z \in D} f(x; z)$. Given access to $D$, the goal is to find an approximate second-order stationary point (SOSP) of the unknown population risk $F(\cdot)$. In general, we have the notion of $(\alpha_g, \alpha_H)$-SOSP:

\begin{definition}[$(\alpha_g, \alpha_H)$-SOSP]
    A point $x$ is an $(\alpha_g, \alpha_H)$-SOSP of a twice differentiable function $F(\cdot)$ if $x$ satisfies $\|\nabla F(x)\|\le\alpha_g$ and $\nabla^2 F(x)\succeq -\alpha_H\cdot \mathbf{I}_d$.
\end{definition}

As shown in~\cite[Proposition 1]{yin2019defending}, there exists a lower bound of $\tilde{O}(\alpha_g^{1/2})$ for $\alpha_H$ given $\alpha_g$, implying that an $(\alpha, \tilde{O}(\sqrt{\alpha}))$-SOSP is the best second-order guarantee achievable. Accordingly, we target the notion of $\alpha$-SOSP in this work, following~\cite{liu2024private}.

\begin{definition}[$\alpha$-SOSP]\label{def-SOSP}
    A point $x$ is an $\alpha$-SOSP of a twice differentiable function $F(\cdot)$ if $x$ satisfies $\|\nabla F(x)\|\le\alpha$ and $\nabla^2 F(x)\succeq -\sqrt{\rho\alpha}\cdot \mathbf{I}_d$.
\end{definition}

An $\alpha$-SOSP excludes $\alpha$-strict saddle points where $\nabla^2 F(x) \preceq -\sqrt{\rho \alpha} \mathbf{I}_d$, thereby ensuring convergence to an approximate local minimum. Following prior work~\cite{liu2024private,jin2021nonconvex}, we assume $M \ge \sqrt{\rho \alpha}$ so that finding an SOSP is strictly more challenging than finding an FOSP.

\noindent{\bf Distributed Learning.} In the distributed (federated) learning setting, $m$ clients collaboratively learn under the coordination of a central server. Each client $j\in[m]$ has a local dataset $D_j$ of size $n$, sampled from an unknown local distribution $\mathcal{D}_j$. The population risk for client $j$ is defined as $F_{\mathcal{D}_j}(x)\coloneqq\mathbb{E}_{z\sim\mathcal{D}j}[f(x;z)]$ or simply $F_j(x)$. The global population risk is defined as the average of the local population risks: $F_{\mathcal{D}}(x)\coloneqq \frac{1}{m}\sum_{j\in[m]}F_j(x)$, or simply $F(x)$. We allow for heterogeneous local datasets, meaning that the local distributions $\{\mathcal{D}_j\}_{j\in[m]}$ may differ.

\noindent{\bf Differential Privacy.} We aim to find an $\alpha$-SOSP under the requirment of Differential Privacy (DP), which is referred to as an $\alpha$-DP-SOSP. We say two datasets $D$ and $D^\prime$ are \textit{adjacent} if they differ by at most one record. DP ensures that the output of the stochastic optimization algorithm on any pair of adjacent datasets is statistically indistinguishable.

\begin{definition}[Differential Privacy (DP)~\cite{dwork2006calibrating}]
    Given $\epsilon, \delta>0$, a randomized algorithm $\mathcal{A}: \mathcal{Z}\to\mathcal{X}$ is ($\epsilon, \delta$)-DP if for any pair of adjacent datasets $D,D^\prime\subseteq\mathcal{Z}$, and any measurable subset $S\subseteq\mathcal{X}$,
    \begin{equation*}\label{eq-dp}
        \mathbb{P}[\mathcal{A}(D)\in S]\le \exp(\epsilon)\cdot\mathbb{P}[\mathcal{A}(D^\prime)\in S]+\delta.
    \end{equation*}
\end{definition}

In distributed learning, we focus on \textit{inter-client record-level DP (ICRL-DP)}, which assumes that clients do not trust the server or other clients with their sensitive local data. This notion has been widely adopted in state-of-the-art distributed learning works, such as \cite{gao2024private,lowy2023privatea,lowy2023privateb}.

\begin{definition}[Inter-Client Record-Level DP (ICRL-DP)]
    Given $\epsilon, \delta > 0$, a randomized algorithm $\mathcal{A}: \mathcal{Z}^m\to\mathcal{X}$ satisfies ($\epsilon, \delta$)-ICRL-DP if, for any client $j\in[m]$ and any pair of local datasets $D_j$ and $D_j^\prime$, the full transcript of client $j$’s sent messages during the learning process satisfies \eqref{eq-dp}, assuming fixed local datasets for other clients.
\end{definition}

\noindent{\bf Variance Reduction via SPIDER.}  
Since the population risk $F(\cdot)$ is unknown, standard SGD approximates the true gradient $\nabla F(x_{t-1})$ at iteration $t$ using a stochastic estimate $g_t$. However, such estimates often exhibit high variance, degrading convergence. The Stochastic Path Integrated Differential Estimator (SPIDER)~\cite{fang2018spider} mitigates this variance using two gradient oracles $\mathcal{O}_1$ and $\mathcal{O}_2$. For a mini-batch $\mathcal{B}_t$ at iteration $t$, we define  
\[
\mathcal{O}_1(x_{t-1}, \mathcal{B}_t) \coloneqq \nabla \hat{f}_{\mathcal{B}_t}(x_{t-1}), \quad
\mathcal{O}_2(x_{t-1}, x_{t-2}, \mathcal{B}_t) \coloneqq \nabla \hat{f}_{\mathcal{B}_t}(x_{t-1}) - \nabla \hat{f}_{\mathcal{B}_t}(x_{t-2}).
\]
SPIDER queries $\mathcal{O}_1$ every $p$ iterations to refresh the gradient estimate. Between these updates, it uses $\mathcal{O}_2$ to incrementally refine the estimate:  
\[
g_t =
\begin{cases}
\mathcal{O}_1(x_{t-1}, \mathcal{B}_t), & \text{if } (t-1) \bmod p = 0, \\
g_{t-1} + \mathcal{O}_2(x_{t-1}, x_{t-2}, \mathcal{B}_t), & \text{otherwise}.
\end{cases}
\]
For smooth functions, the variance of $\mathcal{O}_2(x_{t-1}, x_{t-2}, \mathcal{B}_t)$ scales with $\|x_{t-1} - x_{t-2}\|$, which is typically small when updates are minimal. This allows SPIDER to achieve low-variance gradient estimates while maintaining accuracy.

We choose SPIDER because it achieves state-of-the-art error rates for privately finding first-order stationary points (DP-FOSP)~\cite{arora2023faster}. Our goal is to investigate whether its variance reduction can extend to DP-SOSP. Importantly, the insights in this paper are not specific to SPIDER; they also apply to other variance-reduced methods such as STORM~\cite{cutkosky2019momentum} or SARAH~\cite{nguyen2017sarah}. However, since these algorithms are conceptually similar, no significant improvement is expected from substituting them.

\section{Our Generic Perturbed SGD Framework}\label{sec-framework}

\begin{algorithm}[t]
    \caption{\texttt{Gauss-PSGD}: Gaussian Perturbed Stochastic Gradient Descent}
    \label{algo-psgd}
    \LinesNumbered
    \KwIn{Failure probability $\omega$, initial model $x_0$, learning rate $\eta$, \# of escape repeats $Q$, model deviation threshold $\mathcal{R}$, \# of escape steps $\Gamma$}
    
    \SetKw{Break}{break}
    $t\gets 0$\;
    \While{$\mathrm{true}$}{
        $t\gets t+1$\;
        $\hat{g}_t\gets\texttt{P\_Grad\_Oracle}(*)$\;
        \eIf{$\|\hat{g}_{t}\|\le3\chi$}{
            \tcc{Saddle point escape}
            $\tilde{t}\gets t$, $\tilde{x}\gets x_{t-1}$, $\mathsf{esc}\gets\mathrm{false}$\;
            \For{$q\gets 1,\cdots, Q$}{
                $t\gets\tilde{t}$, $x_t\gets\tilde{x}$\;
                \For{$\tau\gets 1, \cdots, \Gamma$}{
                    $\hat{g}_t\gets\texttt{P\_Grad\_Oracle}(*)$ \;
                    
                    $x_t\gets x_{t-1}-\eta\cdot\hat{g}_t$\;
                    \eIf{$\|x_{t}-\tilde{x}\|\ge\mathcal{R}$}{
                        $\textsf{esc}\gets\mathrm{true}$\;
                        \Break\;
                    }
                    {
                        $t\gets t+1$\;
                    }
                }
                \If{$\mathsf{esc}=\mathrm{true}$}{
                    \Break\;
                }
            }
            \If{$\mathsf{esc}=\mathrm{false}$}{
                \Return $x_{t-1}$
            }
        }
        {
            \tcc{Normal descent step}
            $x_t\gets x_{t-1}-\eta\cdot\hat{g}_{t}$\;
        }
    }
\end{algorithm}

In this section, we introduce a generic framework for finding an $\alpha$-SOSP of the population risk $F_\mathcal{D}(\cdot)$ by escaping saddle points. Our framework is a Gaussian perturbed stochastic gradient descent method, denoted as \texttt{Gauss-PSGD}.

\subsection{Gradient Oracle Setup}

Since $\nabla F_\mathcal{D}(\cdot)$ is unknown, direct gradient descent is infeasible. As in standard stochastic optimization, we assume access to a stochastic gradient oracle $g_t$ that approximates $\nabla F_\mathcal{D}(x_{t-1})$ at iteration $t$. For example, $g_t$ can be computed as an empirical gradient over a mini-batch $\mathcal{B}_t$ sampled from $\mathcal{D}$. We model the oracle as
\begin{equation}
    g_t=\nabla F(x_{t-1})+\zeta_t,
\end{equation}
where $\zeta_t$ represents inherent gradient noise. Following~\cite{jin2021nonconvex, liu2024private}, we assume $\zeta_t \sim \operatorname{nSG}(\sigma)$, where $\operatorname{nSG}$ denotes a norm-sub-Gaussian distribution (Definition~\ref{def-normsg} in Appendix~\ref{sec-facts}).

To enable saddle point escape, we introduce an additional Gaussian perturbation to form a perturbed gradient oracle $\hat{g}_t$:
\begin{equation}\label{eq-oracle}
    \hat{g}_t=g_t+\xi_t=\nabla F(x_{t-1})+\zeta_t+\xi_t,
\end{equation}
where $\xi_t \sim \mathcal{N}(0, r^2 \mathbf{I}_d)$. We define the effective noise magnitude in $\hat{g}_t$ as 
\begin{equation}
    \psi \coloneqq \sqrt{\sigma^2 + r^2 d}.
\end{equation}
The model update is then performed by
\begin{equation}\label{eq-descent}
    x_{t}\gets x_{t-1}-\eta \hat{g}_{t}.
\end{equation}

Our problem setting fundamentally differs from that in~\cite{jin2021nonconvex}. In their setting, the target error $\alpha$ is given, and the perturbation magnitude $r$ is determined accordingly. In contrast, in our privacy-constrained setting, $r$ is dictated by the privacy parameters $(\epsilon, \delta)$, and the goal is to achieve the smallest possible $\alpha$ under this constraint. Crucially, their parameterization $r = O(\sqrt{(\sigma^2 + \alpha^{3/2})/d})$ implies that $r$ depends on both $\sigma$ and $\alpha$, determined by $\max\{\sigma / \sqrt{d}, \alpha^{3/4} / \sqrt{d}\}$. This non-invertible relationship between $r$ and $\alpha$ makes their setting incompatible with ours. First, under DP constraints, $r$ is determined by $(\epsilon, \delta)$ and may be smaller than $\sigma / \sqrt{d}$ in weak privacy regimes, violating the required lower bound. Second, because $r$ and $\alpha$ are not uniquely determined by each other, it is not meaningful to directly translate their error bounds into our setting. Thus, their analysis and results cannot be directly applied to our problem.

\subsection{Our Approach: A General Gaussian-Perturbed SGD Framework}

We present our \texttt{Gauss-PSGD} framework in Algorithm~\ref{algo-psgd}. As specified in~\eqref{eq-oracle}, we employ a general Gaussian-perturbed stochastic gradient oracle, denoted as $\texttt{P\_Grad\_Oracle}(*)$ in steps~4 and 10, where $*$ abstracts the specific arguments required by the oracle implementation. This abstraction allows \texttt{Gauss-PSGD} to serve as a flexible optimization framework for non-convex stochastic problems, applicable beyond the differential privacy (DP) setting.

At each iteration, the gradient estimate $\hat{g}_t$ is computed by $\texttt{P\_Grad\_Oracle}(*)$, and the model parameter is updated via the gradient descent step in~\eqref{eq-descent}. The algorithm proceeds until it encounters a point $\tilde{x}$ satisfying $\|\hat{g}_t\| \le 3\chi$, where $\chi$ is specified in~\eqref{eq-notation}. This point $\tilde{x}$ may lie near a saddle point with a large negative eigenvalue of the Hessian. To escape such a saddle point, the framework enters an escape procedure (steps~6–20), which performs $Q$ rounds of \texttt{$\Gamma$-descent} (steps~9–16).

In each round, the algorithm executes at most $\Gamma$ perturbed SGD iterations starting from $\tilde{x}$. If at any iteration we observe $\|x_t - \tilde{x}\| \ge \mathcal{R}$ for a threshold $\mathcal{R}$ (specified in~\eqref{eq-notation}), indicating that the iterate has moved sufficiently far from $\tilde{x}$, we declare that the algorithm has successfully escaped the saddle point and resume normal PSGD from $x_t$. If no such movement is observed after $Q$ rounds, we declare $\tilde{x}$ an $\alpha$-SOSP of the population risk $F_{\mathcal{D}}(\cdot)$ and output $\tilde{x}$. The repetition over $Q$ rounds ensures a high probability of escape: as we will prove later, each \texttt{$\Gamma$-descent} succeeds in escaping a saddle point with constant probability, and multiple repetitions reduce the failure probability to any desired level.

A central innovation of our framework is using model drift distance as the escape criterion (step~12), replacing the function value decrease criterion used in~\cite{jin2017escape, jin2021nonconvex}. This design enables the algorithm to identify an SOSP with high probability during the optimization process itself, eliminating the need for an auxiliary private model selection step. Our key insight is as follows: escaping a saddle point not only causes a decrease in the objective function~\cite{jin2017escape, jin2021nonconvex} but also induces a substantial displacement of the model parameter beyond a threshold $\mathcal{R}$. Shifting from monitoring function values to tracking parameter movement is critical in population risk settings, where the objective function is unknown and function evaluations are unavailable, unlike in empirical risk minimization~\cite{jin2017escape}. However, the model iterates and their deviations are observable. By leveraging this property, our framework can directly output an SOSP, rather than merely guaranteeing its existence among the iterates.

\subsection{Main Results for \texttt{Gauss-PSGD} Framework}

We begin by introducing the parameter setup and notations used throughout the analysis:

\begin{equation}\label{eq-notation}
\begin{gathered}
    \iota\coloneqq s\mu,\quad 
    \chi\coloneqq 4\sqrt{C}s\mu^2\psi,\quad 
    \alpha\coloneqq 4\chi,\quad\\
    \Gamma\coloneqq\frac{\iota}{s\eta\sqrt{\rho\alpha}},\quad  
    \mathcal{R}\coloneqq\frac{1}{\iota^{1.5}}\sqrt{\frac{\alpha}{\rho}},\quad  
    \Phi\coloneqq\frac{s}{8\iota^3}\sqrt{\frac{\alpha^3}{\rho}},\quad  
    \eta\coloneqq\frac{\sqrt{\rho\alpha}}{M^2\iota^2}.
\end{gathered}
\end{equation}
where $s$ is a sufficiently large absolute constant to be chosen later, and $\mu$ is a logarithmic factor:
\begin{equation}
    \mu\coloneqq\max\left\{\frac{1}{s}\log\left(\frac{9d\log\left(\frac{4C^{1/4}}{s\eta r}\sqrt{\frac{\psi}{\rho}}\right)}{C^{1/4}\eta\sqrt{s\rho\psi}}\right), \log\left(\frac{160\sqrt{2}C^{1/4}}{s\sqrt{\eta r}}\sqrt{\frac{\psi}{\rho}}\right),
    \frac{\left(C\log \frac{4T}{\omega}\right)^{1/4}}{2^\frac{3}{4}\sqrt{s}},
    1\right\}.
\end{equation}
Here $C$ is an absolute constant that may change across expressions. Let $\tilde{x}$ denote a saddle point of the population risk $F(\cdot)$, and $\mathcal{H} \coloneqq \nabla^2 F(\tilde{x})$. Let $v_{\min}$ be the eigenvector corresponding to $\lambda_{\min}(\mathcal{H})$, and $\mathcal{P}_{-v_{\min}}$ be the projection onto the orthogonal complement of $v_{\min}$. Set $\gamma \coloneqq -\lambda_{\min}(\mathcal{H})$.

\begin{definition}[Coupling Sequence]
    Let $\{x_i\}$ and $\{x_i^\prime\}$ be two PSGD sequences initialized at $\tilde{x}$. We say they are \emph{coupled} if they share the same randomness for $\mathcal{P}_{-v_{\min}} \xi_t$ and $\zeta_t$ at each iteration $t$, but use opposite perturbations in the $v_{\min}$ direction: $v_{\min}^\top \xi_t = -v_{\min}^\top \xi_t^\prime$.
\end{definition}

The following lemma ensures that under \texttt{$\Gamma$-descent}, at least one of the coupled sequences escapes the saddle point with constant probability (proof in Appendix~\ref{sec-le-escape}).

\begin{lemma}[Escaping Saddle Points]\label{le-escape}
Let $\{x_i\}$ and $\{x_i^\prime\}$ be coupled PSGD sequences initialized at $\tilde{x}$ such that $\|\nabla F(\tilde{x})\| \le \alpha$ and $\lambda_{\min}(\nabla^2 F(\tilde{x})) \le -\sqrt{\rho \alpha}$. Then, with probability at least $1/4$, there exists $\tau \le \Gamma$ such that $\max\left\{\|x_\tau - \tilde{x}\|, \|x_\tau^\prime - \tilde{x}\|\right\} \ge \mathcal{R}.$
\end{lemma}

From this, we immediately obtain a corollary that applies to any PSGD sequence:

\begin{corollary}\label{co-escape}
    For any PSGD sequence $\{x_i\}$ starting at $\tilde{x}$ with $\|\nabla F(\tilde{x})\| \le \alpha$ and $\lambda_{\min}(\nabla^2 F(\tilde{x})) \le -\sqrt{\rho \alpha}$, with probability at least $1/8$, there exists $t \le \Gamma$ such that $\|x_t - \tilde{x}\| \ge \mathcal{R}$.
\end{corollary}

To ensure a high-probability escape from a saddle point, we repeat \texttt{$\Gamma$-descent} for $Q$ rounds:

\begin{lemma}[Escape Amplification via Repetition]\label{le-repeat}
Given any $\omega_0\in(0,1)$, repeating \texttt{$\Gamma$-descent} independently for $Q = \frac{26}{5} \log(\frac{1}{\omega_0})$ rounds ensures escape with probability at least $1 - \omega_0$.
\end{lemma}

The proof is deferred to Appendix~\ref{sec-le-repeat}. We now analyze the total number of PSGD steps needed for convergence. Let $\nu_t \coloneqq \zeta_t + \xi_t$ denote the combined noise in the gradient estimate.

\begin{lemma}[Descent Lemma]\label{le-descent}
    For any $t_0$, the following holds:
    \begin{equation}
        F(x_{t_0+t})-F(x_{t_0})\!\le\! -\frac{\eta}{2}\sum_{i=0}^{t-1}\|\nabla F(x_{t_0+i})\|^2+\frac{\eta}{2}\sum_{i=1}^{t}\|\nu_{t_0+i}\|^2
    \end{equation}
\end{lemma}

Since $\nu_t$ can be bounded with high probability, we have:

\begin{corollary}\label{co-descent}
    For any $t_0$ and some constant $c$, with probability at least $1 - 2e^{-\iota}$,
    \begin{equation}
        F(x_{t_0+t})-F(x_{t_0})\le-\frac{\eta}{2}\sum_{i=0}^{t-1}\|\nabla F(x_{t_0+i})\|^2+c\eta\psi^2(t+\iota).
    \end{equation}
\end{corollary}

Proofs of Lemma~\ref{le-descent} and Corollary~\ref{co-descent} are in Appendix~\ref{sec-le-descent} and~\ref{sec-co-descent}. These imply that large gradients lead to rapid function decrease. We next show in Lemma~\ref{le-decrease} that a successful saddle point escape via \texttt{$\Gamma$-descent} leads to a significant decrease in function value, whose proof is in Appendix~\ref{sec-le-decrease}.

\begin{lemma}[Value Decrease per Escape]\label{le-decrease}
    Let a \texttt{$\Gamma$-descent} starting from $x_{t_0}$ succeed after $\tau \le \Gamma$ steps. With probability at least $1 - 2e^{-\iota}$, $F(x_{t_0+\tau})-F(x_{t_0})\le-\frac{s}{8\iota^3}\sqrt{\frac{\alpha^3}{\rho}}=-\Phi$.
\end{lemma}

We bound the total number of PSGD steps required for convergence, based on the following estimate:

\begin{lemma}[Gradient Estimate Error Bound]\label{le-grad-est-error}
With probability at least $1 - \omega/2$, for all $t \in [T]$, $\|\nu_t\|\le C\sqrt{2\log\left(\frac{4T}{\omega}\right)}\psi\le\chi$.
\end{lemma}

\begin{lemma}[Maximum Number of Descent Steps]\label{le-steps}
    Given failure probability $\omega$, set $Q\!=\!\frac{26}{5}\log\!\left(\!\frac{16\iota^3(F_0-F^*)}{s\omega}\!\sqrt{\frac{\rho}{\chi^3}}\right)$. \texttt{Gauss-PSGD} returns an $\alpha$-SOSP within at most $\tilde{O}(1/\alpha^{2.5})$ PSGD steps.
\end{lemma}

Proofs of Lemmas~\ref{le-grad-est-error} and~\ref{le-steps} are in Appendix~\ref{sec-le-grad-est-error} and~\ref{sec-le-steps}, respectively.

\begin{remark}[On Gradient Complexity]\label{rm-grad-complexity}
While Lemma~\ref{le-steps} appears to improve gradient complexity from $O(1/\alpha^4)$ in~\cite{jin2021nonconvex} to $O(1/\alpha^{2.5})$, the two results are not directly comparable. In~\cite{jin2021nonconvex}, the error target $\alpha$ is treated as an input and can be arbitrarily small, with gradient variance $\sigma$ typically treated as a constant. In contrast, in our setting, the perturbation $r$ and variance $\sigma$ are fixed by privacy constraints, and $\alpha$ emerges as a function of these. Thus, our gradient complexity fundamentally depends on $\sigma$ and $r$, though we express it in terms of $\alpha$ for clarity.
\end{remark}

Combining all the above, we obtain the final convergence guarantee:

\begin{theorem}[Convergence Guarantee of \texttt{Gauss-PSGD}]\label{th-framework}
Let Assumptions~\ref{assump-loss} and~\ref{assump-func} hold. For any failure probability $\omega \in (0,1)$, using the parameter settings in~\eqref{eq-notation} and setting $Q = \frac{26}{5} \log\left(\frac{16\iota^3 (F_0 - F^*)}{s\omega} \sqrt{\frac{\rho}{\chi^3}}\right)$, then with probability at least $1 - \omega$, \texttt{Gauss-PSGD} (Algorithm~\ref{algo-psgd}) returns an $\alpha$-SOSP of $F(\cdot)$, where $\alpha = 4\chi$, within at most $\tilde{O}(1/\alpha^{2.5})$ PSGD steps.
\end{theorem}

\section{Rectified Error Rate for finding SOSP in DP Stochastic Optimization}\label{sec-DP}

\subsection{Adaptive Gradient Oracle: \texttt{Ada-DP-SPIDER}}

In this section, we derive the upper bound on the error rate for DP stochastic optimization by instantiating the \texttt{Gauss-PSGD} framework with a specific gradient oracle. We adopt an adaptive version of the DP-SPIDER estimator, referred to as \texttt{Ada-DP-SPIDER}, which is presented in Algorithm~\ref{algo-adpt-spider}. This adaptive version refines the original SPIDER by dynamically adjusting gradient queries based on model drift. Unlike standard SPIDER, which queries $\mathcal{O}_1$ at fixed intervals and may suffer from growing estimation error over time, \texttt{Ada-DP-SPIDER} tracks the cumulative model drift defined as
\begin{equation}
    \textsf{drift}_t\coloneqq\sum_{i=\tau(t)}^t\|x_i-x_{i-1}\|^2,
\end{equation}
where $\tau(t)$ is the last iteration at which the full gradient oracle $\mathcal{O}_1$ was queried. 

The intuition is that, for smooth functions, the error of $\mathcal{O}_2$, which estimates $\nabla F(x_{t-1}) - \nabla F(x_{t-2})$, is proportional to $\|x_{t-1} - x_{t-2}\|$. When the model drift is small, $\mathcal{O}_2$ remains accurate, allowing for continued use to reduce variance (steps~9-11). However, when the drift becomes large, further use of $\mathcal{O}_2$ can accumulate significant errors. To mitigate this, the algorithm triggers a fresh query to $\mathcal{O}_1$ (steps~4-7). A threshold $\kappa$ is used in step~3 to determine when the drift is large. This enables adaptive switching between oracles based on the model drift, ensuring the total error remains well controlled.

Our approach differs fundamentally from that of~\cite{liu2024private}. In their method, in addition to using model drift to trigger $\mathcal{O}_1$, they also invoke $\mathcal{O}_1$ when approaching potential saddle points and inject an additional Gaussian noise on top of the DP gradient estimator to escape. To prevent excessive noise injection, they introduce a \textsf{Frozen} state to restrict how frequently this occurs. In contrast, our method leverages the inherent Gaussian noise from the DP gradient estimator for saddle point escape and uses model drift as the sole trigger for querying $\mathcal{O}_1$. This results in a simpler, more efficient estimator without auxiliary state tracking or redundant noise injection.

\subsection{Error Rate Analysis for DP-SOSP with \texttt{Ada-DP-SPIDER}}

To minimize the error rate $\alpha$ for DP-SOSP using \texttt{Ada-DP-SPIDER}, we must carefully tune algorithmic parameters, including the mini-batch sizes $b_1$, $b_2$, and the drift threshold $\kappa$. These parameters directly influence the gradient estimation error, which, according to Theorem~\ref{th-framework}, dominates the learning error. The following lemma characterizes how these parameters affect the estimation quality:

\begin{lemma}\label{le-adpt-spider-error}
Let Assumption~\ref{assump-loss} hold. For all $t \in [T]$, the gradient estimate $\hat{g}_t$ given by \texttt{Ada-DP-SPIDER} satisfies: $\sigma \le O\left(\sqrt{\frac{G^2 \log^2 d}{b_1} + \frac{M^2 \log^2 d}{b_2} \kappa} \right), r \le O\left( \sqrt{ \frac{G^2 \log(1/\delta)}{b_1^2 \epsilon^2} + \frac{M^2 \log(1/\delta)}{b_2^2 \epsilon^2} \kappa } \right)$.
\end{lemma}

The proof is given in Appendix~\ref{sec-le-adpt-spider-error}. To ensure that $b_1$ and $b_2$ remain valid mini-batch sizes under a fixed sample budget, we must control the number of times $\mathcal{O}_1$ is queried. Lemma~\ref{le-driftnum} bounds the count:

\begin{lemma}\label{le-driftnum}
    Let Assumption~\ref{assump-loss} and~\ref{assump-func} hold. Define $\mathcal{T}\coloneqq\{t\in[T]: \operatorname{drift}_{t}\ge\kappa\}$ as the set of rounds where the drift exceeds the threshold $\kappa$. With high probability (as in Theorem~\ref{th-framework}), $|\mathcal{T}|\le O\left(U\eta/\kappa\right)$.
\end{lemma}

Proof is in Appendix~\ref{sec-le-driftnum}. Guided by Lemmas~\ref{le-adpt-spider-error} and~\ref{le-driftnum}, we now derive the error bound for $\alpha$ via appropriate choices of $b_1$, $b_2$, and $\kappa$ in Theorem~\ref{th-adapt-spider}. The proof is provided in Appendix~\ref{sec-th-adapt-spider}.

\begin{theorem}\label{th-adapt-spider}
    Let Assumption~\ref{assump-loss} and~\ref{assump-func} hold. Define $b_1=\frac{n\kappa}{2U\eta}$, $b_2=\frac{n\eta\chi^2}{2U}$ and $\kappa=\max\left\{\frac{G^{3/2}U^{1/2}\rho^{1/2}}{M^{5/2}n^{1/2}}, \frac{G^{14/15}d^{2/5}U^{4/5}\rho^{8/15}}{M^{34/15}(n\epsilon)^{4/5}}\right\}$. Then, running \texttt{Gauss-PSGD} with gradient oracle instantiated by \texttt{Ada-DP-SPIDER} ensures $(\epsilon,\delta)$-DP for constants $c_1, c_2$ and returns an $\alpha$-SOSP with $\alpha=\tilde{O}\left(\frac{1}{n^{1/3}} + \left(\frac{\sqrt{d}}{n\epsilon}\right)^{2/5}\right)$.
\end{theorem}

\begin{remark}[No Cyclic Dependency Among Parameters]
All algorithmic parameters are consistently defined in terms of the problem parameters $n$, $d$, and $\epsilon$. Specifically, \texttt{Gauss-PSGD} parameters such as the step size $\eta$ and the noise scale $\chi$ depend on the target error $\alpha$ (see \eqref{eq-notation}), and the gradient oracle parameters $b_1$ and $b_2$ are defined through $\eta$ and $\chi$, and thus also indirectly depend on $\alpha$. In the proof of Theorem~\ref{th-adapt-spider}, by utilizing the relationship $\alpha = \tilde{O}(\sqrt{\sigma^2 + r^2 d})$, we obtain the closed-form expression of $\alpha$ that depends solely on the problem parameters $n$, $d$, and $\epsilon$. As a result, all algorithm parameters are ultimately determined by $n$, $d$, and $\epsilon$, and there is no cyclic dependency in the parameter design.
\end{remark}

\begin{algorithm}[t]
\caption{\texttt{Ada-DP-SPIDER}}
\label{algo-adpt-spider}

\KwIn{DP budget $\epsilon$ and $\delta$, horizon $T$, model iterates $\{x_{t-1}\}_{t=1}^T$, drift threshold $\kappa$}

$t\gets 1$, $\textsf{drift}\gets\kappa$\;
\While{$t\le T$}{
    \eIf{$\textsf{drift}\ge\kappa$}{
        \tcc{Using oracle $ \mathcal{O}_1$}
        Sample mini-batch $\mathcal{B}_t$ of size $b_1$ from $\mathcal{D}$\;
        Sample $\xi_t\sim\mathcal{N}(0, c_1\frac{G^2\log\frac{1}{\delta}}{b_1^2\epsilon^2}\mathbf{I}_d)$\;
        $\hat{g}_{t}\gets\mathcal{O}_1(x_{t-1}, \mathcal{B}_t)+\xi_t$\;
        $\textsf{drift}\gets 0$\;
    }{
        \tcc{Using oracle $\mathcal{O}_2$}
        Sample mini-batch $\mathcal{B}_t$ of size $b_2$ from $\mathcal{D}$\;
        Sample $\xi_t\sim\mathcal{N}(0,$ $c_2\frac{M^2\log\frac{1}{\delta}}{b_2^2\epsilon^2}\|x_{t\!-\!1}\!-\!x_{t\!-\!2}\|^2\mathbf{I}_d)$\;
        $\hat{g}_{t}\gets\hat{g}_{t-1}+\mathcal{O}_2(x_{t-1}, x_{t-2}, \mathcal{B}_t)+\xi_t$\;
    }
    $\textsf{drift}\gets \textsf{drift}+\eta^2\|\hat{g}_t\|^2$\;
    $t\gets t+1$\;
}
\KwOut{$\hat{g}_1, \hat{g}_2, \cdots, \hat{g}_T$}
\end{algorithm}

\begin{algorithm}[t]
\caption{Distributed \texttt{Ada-DP-SPIDER}}
\label{algo-dis-adpt-spider}

\KwIn{DP budget $\epsilon$ and $\delta$, horizon $T$, model iterates $\{x_{t-1}\}_{t=1}^T$, drift threshold $\kappa$}

$t\gets 1$, $\textsf{drift}\gets\kappa$\;
\While{$t\le T$}{
    \eIf{$\textsf{drift}\ge\kappa$}{
        \For{\textbf{\textup{every}} \textbf{client} $j$ \textbf{\textup{in parallel}}}{
            Sample mini-batch $\mathcal{B}_{j,t}$ of size $b_1$ from $\mathcal{D}_j$\;
            Sample $\xi_{j,t}\sim\mathcal{N}(0, c_1\frac{G^2\log\frac{1}{\delta}}{b_1^2\epsilon^2}\mathbf{I}_d)$\;
            $\hat{g}_{j,t}\gets\mathcal{O}_1(x_{t-1}, \mathcal{B}_{j,t})+\xi_{j,t}$\;
            Send $\hat{g}_{j,t}$ to the server\;
        }        $\textsf{drift}\gets 0$\;
    }{
        \For{\textbf{\textup{every}} \textbf{client} $i$ \textbf{\textup{in parallel}}}{
            Sample mini-batch $\mathcal{B}_{j,t}$ of size $b_2$ from $\mathcal{D}_j$\;
            Sample $\xi_{j,t}\sim\mathcal{N}(0,$ $c_2\frac{M^2\log\frac{1}{\delta}}{b_2^2\epsilon^2}\|x_{t\!-\!1}\!-\!x_{t\!-\!2}\|^2\mathbf{I}_d)$\;
            $\hat{g}_{j,t}\gets\hat{g}_{j,t-1}+\mathcal{O}_2(x_{t-1}, x_{t-2}, \mathcal{B}_{j,t})+\xi_{j,t}$\;
            Send $\hat{g}_{j,t}$ to the server\;
        }
    }
    $\hat{g}_t\gets\frac{1}{m}\sum_{j=1}^m \hat{g}_{j,t}$\;
    $\textsf{drift}\gets \textsf{drift}+\eta^2\|\hat{g}_t\|^2$\;
    $t\gets t+1$\;
}
\KwOut{$\hat{g}_1, \hat{g}_2, \cdots, \hat{g}_T$}
\end{algorithm}

\section{Extension to Distributed SGD}\label{sec-distributed}

%\subsection{Learning with Distributed \texttt{Ada-DP-SPIDER} and Convergence Guarantee}

By adapting the centralized gradient oracle \texttt{Ada-DP-SPIDER} (Algorithm~\ref{algo-adpt-spider}) to the distributed setting, we obtain \texttt{Distributed Ada-DP-SPIDER} (Algorithm~\ref{algo-dis-adpt-spider}), enabling our \texttt{Gauss-PSGD} framework to extend seamlessly to distributed learning scenarios. The primary difference lies in the computation and communication scheme: in the distributed variant, each client performs local gradient estimation with private noise and communicates the privatized estimate to the server, which then aggregates the results. This avoids centralized access to raw data while still leveraging collective information.

The learning algorithm using \texttt{Distributed Ada-DP-SPIDER} can be viewed as an adaptive extension of the DIFF2 algorithm~\cite{murata2023diff2}, which uses standard SPIDER and is limited to convergence to DP-FOSP under \textit{homogeneous} data. To the best of our knowledge, our method is the first to achieve convergence to a DP-SOSP in a distributed setting with \textit{heterogeneous} data.

Following the same analytical strategy as in Section~\ref{sec-DP}, we first quantify in Lemma~\ref{le-dis-adpt-spider-error} the gradient estimation quality in the distributed case. The proof is provided in Appendix~\ref{sec-le-dis-adpt-spider-error}.

\begin{lemma}\label{le-dis-adpt-spider-error}
    Let Assumption~\ref{assump-loss} hold. For all $t\in[T]$, the distributed \texttt{Ada-DP-SPIDER} ensures that the gradient estimate $\hat{g}_t$ satisfies $\sigma\le O\left(\sqrt{\frac{G^2\log^2 d}{m\cdot b_1}+\frac{M^2\log^2 d}{m \cdot b_2}\kappa}\right),  r\le O\left(\sqrt{\frac{G^2\log\frac{1}{\delta}}{m\cdot b_1^2\epsilon^2}+\frac{M^2\log\frac{1}{\delta}}{m\cdot b_2^2\epsilon^2}\kappa}\right)$.
\end{lemma}

Based on this, we derive the error bound for $\alpha$ in the distributed setting. The proof is in Appendix~\ref{sec-th-dis-adapt-spider}.

\begin{theorem}\label{th-dis-adapt-spider}
    Let Assumption~\ref{assump-loss} and~\ref{assump-func} hold. Define $b_1=\frac{n\kappa}{2U\eta}$, $b_2=\frac{n\eta\chi^2}{2U}$ and $\kappa=\max\left\{\frac{G^{3/2}U^{1/2}\rho^{1/2}}{M^{5/2}(mn)^{1/2}}, \frac{G^{14/15}d^{2/5}U^{4/5}\rho^{8/15}}{M^{34/15}(\sqrt{m}n\epsilon)^{4/5}}\right\}$. Then, running \texttt{Gauss-PSGD} with gradient oracle instantiated by distributed \texttt{Ada-DP-SPIDER} ensures $(\epsilon,\delta)$-ICRL-DP for some constants $c_1, c_2$, and returns an $\alpha$-SOSP with $\alpha=\tilde{O}\left(\frac{1}{(mn)^{1/3}}+\left(\frac{\sqrt{d}}{\sqrt{m}n\epsilon}\right)^{2/5}\right)$.
\end{theorem}

\begin{remark}
    The error rate shown in Theorem~\ref{th-dis-adapt-spider} highlights the collaborative synergy among clients, indicating the learning performance benefits from distributed learning. Specifically, the first non-private term of $\alpha$ exhibits a linear dependence on $m$ before $n$, while the second term, which accounts for the privacy cost, demonstrates a square root dependence $\sqrt{m}$ before $n$. This separation reflects the impact of data heterogeneity in distributed setting. The benefit of distributed collaboration under DP constraints is consistent with prior results in heterogeneous federated learning~\cite{gao2024private}.
\end{remark}

\begin{algorithm}[t]
    \caption{Private Model Selection in Distributed Learning}
    \label{algo-select}
    \LinesNumbered
    \KwIn{Model iterates $\{x_t\}_{t=1}^T$, DP budget $\epsilon, \delta$}
    \For{$t\gets 1, \cdots, T$}{
        \For{\textbf{\textup{every}} \textbf{client} $j$ \textbf{\textup{in parallel}}}{
            Compute $\nabla \bar{F}_j(x_t)\gets \nabla \hat{f}_{S_j}(x_t)+\theta_{i,t}$, where $\theta_{i,t}\sim\mathcal{N}\left(0, c_1 \frac{G^2T\log({1}/{\delta})}{n^2\epsilon^2}\mathbf{I}_d\right)$ \;
            Compute $\nabla^2 \bar{F}_j(x_t)\gets \nabla^2 \hat{f}_{S_j}(x_t)+\mathbf{H}_{j,t}$, where $\mathbf{H}_{j,t}$ is a symmetric matrix with its upper triangle (including the diagonal) being i.i.d. samples from $\mathcal{N}\left(0, c_2\frac{M^2d T\log({1}/{\delta})}{n^2\epsilon^2}\right)$ and each lower triangle entry is copied from its upper triangle counterpart\;
            Send $\nabla \bar{F}_j(x_t)$ and $\nabla^2 \bar{F}_j(x_t)$ to the server\;
        }
        $\nabla \bar{F}(x_t)\gets\frac{1}{m}\sum_{j=1}^m \nabla \bar{F}_j(x_t)$, $\nabla^2 \bar{F}(x_t)\gets\frac{1}{m}\sum_{j=1}^m \nabla^2 \bar{F}_j(x_t)$\;
        \If{$\|\nabla \bar{F}(x_t)\|_2\le \alpha +\frac{G\log\left({8d}/{\omega^\prime}\right)}{\sqrt{mn}}+ \frac{G\sqrt{dT\log\left({1}/{\delta}\right)\log\left({16}/{\omega^\prime}\right)}}{\sqrt{m}n\epsilon}$ \normalfont{\textbf{and}} $\lambda_{\min}\left(\nabla^2 \bar{F}(x_t)\right)\ge-\left(\sqrt{\rho \alpha}+M\sqrt{\frac{\log\left({8d}/{\omega^\prime}\right)}{mn}}+\frac{M d\sqrt{T\log({1}/{\delta})\log\left({32}/{\omega^\prime}\right)}}{\sqrt{m}n\epsilon}\right)$}{
            \textbf{Return} $x_{t}$
        }
    }
\end{algorithm}

%\subsection{Superiority of \texttt{Gauss-PSGD} Through Eliminating Private Model Selection}

We conclude by demonstrating the advantages of our \texttt{Gauss-PSGD} framework in distributed learning by eliminating the need for a separate private model selection procedure. Without the guarantee of directly outputting an $\alpha$-SOSP, one must resort to evaluating all model iterates generated during the learning process and privately selecting an approximate SOSP from them. As discussed in Appendix~\ref{sec-gap}, the AboveThreshold mechanism used in~\cite{liu2024private} for the single-machine case is not applicable in distributed settings due to decentralized data access. To overcome this, we adapt~\cite[Algorithm 5]{wang2019differentially} to the distributed setting, resulting in Algorithm~\ref{algo-select}. In this scheme, each client computes privatized gradients and Hessian estimates using additional local data, which are then aggregated by the server to evaluate the stationary point conditions. Suppose a distributed learning algorithm produces a sequence $\{x_t\}_{t \in [T]}$ that contains at least one $\alpha$-DP-SOSP. The following result characterizes the quality of the point selected by Algorithm~\ref{algo-select}, whose proof is provided in Appendix~\ref{sec-th-select}:

\begin{theorem}\label{th-select}
Algorithm~\ref{algo-select} satisfies $(\epsilon, \delta)$-ICRL-DP. Let Assumption~\ref{assump-loss} hold and $mn \ge \frac{4}{9} \log \frac{8d}{\omega^\prime}$, then with probability at least $1 - \omega'$, if there exists an $\alpha$-SOSP $x_p \in \{x_t\}_{t = 1}^T$, then the selected point $x_o$ is an $\alpha^\prime$-SOSP with $\alpha^\prime\!=\!\tilde{O}\left(\!\alpha\!+\!\frac{1}{mn}\!+\!\frac{1}{\sqrt{mn}}\!+\!\frac{\alpha}{\sqrt{mn}}\!+\!\frac{\sqrt{d}}{\sqrt{m}n\epsilon\alpha^{5/4}}\!+\!\frac{d}{\sqrt{m}n\epsilon\alpha^{3/4}}\!+\!\frac{d^2}{mn^2\epsilon^2\alpha^{5/2}}\!\right)$.
\end{theorem}

\begin{remark}
To ensure that the selected model's error $\alpha^\prime$ does not exceed the training error $\alpha$, the following must hold: $\frac{\sqrt{d}}{\sqrt{m}n\epsilon\alpha^{5/4}} \!+\! \frac{d}{\sqrt{m}n\epsilon\alpha^{3/4}} \!+\! \frac{d^2}{mn^2\epsilon^2\alpha^{5/2}}\!\le\! \tilde{O}(\alpha)$. This implies a constraint on the model dimension: $d\le\min\{(\sqrt{m}n\epsilon)^2, (\sqrt{m}n\epsilon)^{6/13}\}$. Thus, in high-dimensional regimes, private model selection degrades the overall error rate, marking the limitation of selection-based approaches.
\end{remark}

\begin{remark}
     The error bound $\alpha^\prime$ in Theorem~\ref{th-select} can be improved by estimating the smallest eigenvalue of the Hessian via Hessian-vector products using iterative methods such as the power method~\cite{lanczos1950iteration}. This reduces the dimensional dependence in the noise scale from $O(d)$ to $O(\sqrt{d})$. However, the remaining $\sqrt{d}$ factor is sill problematic in high-dimensional settings. In contrast, in the single-machine case, private model selection only requires perturbing scalar quantities, making the error independent of dimension, preserving the error guarantee of the learning algorithm. In distributed settings, sharing perturbed vectors becomes unavoidable. This emphasizes the necessity and superiority of our \texttt{Gauss-PSGD} framework that inherently avoids the need for any separate model selection step.
\end{remark}
\section{Limitation Discussion}\label{sec-limitation}

One of the primary objective of this work is to rectify a key analytical error in~\cite{liu2024private} by presenting the correct error rates for DP stochastic non-convex optimization. Our proposed framework, \texttt{Gauss-PSGD}, is designed to be broadly applicable beyond the DP setting, offering a versatile optimization tool for general non-convex problems. Furthermore, this work makes the first attempt to extend DP-SOSP analysis to the distributed learning setting, establishing state-of-the-art utility guarantees.

To maintain consistency with prior work~\cite{liu2024private}, we assume access to an unbiased gradient oracle. This assumption is fundamental in theoretical analysis and is also adopted by many recent studies in DP optimization and distributed learning, such as~\cite{arora2023faster, gao2024private}. However, it may not fully reflect the behavior of practical optimizers that employ biased and noisy gradients, particularly those using gradient clipping---a standard technique in DP implementations.

Nevertheless, our \texttt{Gauss-PSGD} framework can be extended to handle biased oracles induced by clipping. The main challenge lies in the analysis: incorporating clipping introduces bias, requiring a refined characterization of the descent dynamics. In particular, Lemma~\ref{le-descent} (the descent lemma) must be adapted to reflect the bias–variance trade-off. Techniques for bias reduction in clipped DP learning—such as those developed in~\cite{xiao2023theory}---could offer a promising foundation for such an extension.

The saddle point escaping analysis (Lemma~\ref{le-escape}) can also be generalized. As shown in our proof, the key mechanism enabling escape is the injection of symmetric Gaussian noise, which drives the divergence in the coupling sequence. This mechanism remains valid under clipping, provided the Gaussian noise is appropriately calibrated. However, the number of steps required for escape may change due to the altered noise structure and bias, and a more delicate analysis would be required to quantify this behavior accurately.

We consider this as a promising direction for future work and leave its full exploration to subsequent studies.

\section{Conclusion}

In this work, we investigated the problem of finding second-order stationary points (SOSP) in differentially private (DP) stochastic non-convex optimization. We proposed a novel framework that leverages perturbed stochastic gradient descent (SGD) with Gaussian noise and introduces a novel criterion based on model drift distance to ensure provable saddle point escape and efficient convergence. By incorporating an adaptive SPIDER as the gradient oracle, we developed a new DP algorithm that rectifies existing error rates. Furthermore, we extended our approach to distributed learning scenarios with heterogeneous data, providing the first theoretical guarantees for finding DP-SOSP in such settings. Through rigorous analysis, we demonstrated that our framework not only avoids the pitfalls of private model selection but also remains effective in high-dimensional distributed learning environments.

Our work opens several promising directions for future research. A key challenge is bridging the gap between our upper bound and the existing DP lower bound for stochastic optimization, as established in \cite{arora2023faster}. The current lower bound is derived from convex loss functions and first-order stationary points, wheras finding DP-SOSP in non-convex optimization is inherently more difficult. We conjecture that the existing lower bound is not tight for the non-convex case. Establishing a tighter lower bound remains a critical open problem. Additionally, exploring whether our upper bounds can be further improved is another intriguing direction that warrants in-depth investigation.

\bibliographystyle{plain}
\bibliography{main}

%%%%%%%%%%%%%%%%%%%%%%%%%%%%%%%%%%%%%%%%%%%%%%%%%%%%%%%%%%%%
\newpage
\appendix

\section{Useful Facts for Analysis}\label{sec-facts}

\subsection{Probability Tools}

\begin{definition}[{Sub-Gaussian random vector~\cite[Definition~2]{jin2019short}}]\label{def-sg}
    A random vector $v\in\mathbb{R}^d$ is $\zeta$-\textit{sub-Gaussian} (or $\operatorname{SG}(\zeta)$), if there exists a positive constant $\zeta$ such that
    \begin{equation}
        \mathbb{E}[\exp(\langle u, v-\mathbb{E} [v]\rangle)]\le\exp\left(\frac{\|u\|_2^2\zeta^2}{2}\right), \qquad\forall u\in\mathbb{R}^d.
    \end{equation}
\end{definition}

\begin{definition}[{Norm-sub-Gaussian random vector~\cite[Definition~3]{jin2019short}}]\label{def-normsg}
    A random vector $v\in\mathbb{R}^d$ is $\zeta$-\textit{norm-sub-Gaussian} (or $\operatorname{nSG}(\zeta)$), if there exists a positive constant $\zeta$ such that
    \begin{equation}
        \mathbb{P}\left[\left\|v-\mathbb{E}[v]\right\|\ge t\right]\le2\exp\left(-\frac{t^2}{2\zeta^2}\right), \qquad\forall t\in\mathbb{R}.
    \end{equation}
\end{definition}

Note that norm-sub-Gaussian random vectors (Definition~\ref{def-normsg}) are more general than sub-Gaussian random vectors (Definition~\ref{def-sg}), as sub-Gaussian distributions require \textit{isotropy}, whereas norm-sub-Gaussian distributions do not impose this condition.

\begin{lemma}[{\cite[Lemma~1]{jin2019short}}]\label{le-gauss-to-normSG}
    A $\operatorname{SG}(r)$ random vector $v\in\mathbb{R}^d$ is also $\operatorname{nSG}(2\sqrt{2}\cdot r\sqrt{d})$.
\end{lemma}

We are interested in the properties of norm-subGaussian martingale difference sequences. Concretely, they are sequences satisfying the following properties.

\begin{condition}\label{con-martingale}
    Consider random vectors $v_1, \cdots, v_p\in\mathbb{R}^d$, and corresponding filtrations $\mathcal{F}_i=\sigma(v_1, \cdots, v_i)$ for $i\in[n]$, such that $v_i|\mathcal{F}_{i-1}$ is zero-mean $\operatorname{nSG}(\zeta_i)$ with $\zeta_i\in\mathcal{F}_{i-1}$. That is,
    \begin{equation}
        \mathbb{E}[v_i|\mathcal{F}_{i-1}]=0, 
        \qquad \mathbb{P}\left[\left\|v_i\right\|\ge t|\mathcal{F}_{i-1}\right]\le2\exp\left(-\frac{t^2}{2\zeta^2}\right), \qquad\forall t\in\mathbb{R}, \forall i\in[p].
    \end{equation}
\end{condition}

\begin{lemma}[{Hoeffding type inequality for norm-sub-Gaussian~\cite[Corollary 7]{jin2019short}}]\label{le-hoeffding}
    Let random vectors $v_1, \cdots, v_p\in\mathbb{R}^d$, and corresponding filtrations $\mathcal{F}_i=\sigma(v_1, \cdots, v_i)$ for $i\in[k]$ satisfy condition~\ref{con-martingale} with fixed $\{\zeta_i\}$. Then for any $\iota>0$, there exists an absolute constant $C$ such that, with probability at least $1-2d\cdot e^{-\iota}$,
    \begin{equation}
        \left\|\sum_{i=1}^p v_i\right\|_2\le C\cdot\sqrt{\sum_{i=1}^p\zeta_i^2\cdot\iota}.
    \end{equation}
\end{lemma}

Lemma~\ref{le-hoeffding} implies that the sum of norm-sub-Gaussian random vectors is till norm-sub-Gaussian.

\begin{corollary}\label{co-nsg-composition}
    Let random vectors $v_1, \cdots, v_p\in\mathbb{R}^d$, and corresponding filtrations $\mathcal{F}_i=\sigma(v_1, \cdots, v_i)$ for $i\in[k]$ satisfy condition~\ref{con-martingale} with fixed $\{\zeta_i\}$. Then $\sum_{i=1}^p v_i$ is $\operatorname{nSG}\left(C\cdot\sqrt{\log(d)\sum_{i=1}^k\zeta_i^2}\right)$.
\end{corollary}

\begin{proof}
    Let $\zeta_+\coloneqq\sqrt{C\log(d)\sum_{i=1}^k\zeta_i}$. According to Definition~\ref{def-normsg}, we aim to show that, for any $\omega\in(0,1)$, with probability at least $1-\omega$, $\|\sum_{i=1}^p v_i\|\le\sqrt{2\zeta_+^2\ln\frac{2}{\omega}}$. By Lemma~\ref{le-hoeffding}, we have known that, with probability at least $1-\omega$, $\|\sum_{i=1}^p v_i\|\le C\cdot\sqrt{\sum_{i=1}^p\zeta_i^2\ln\frac{2d}{\omega}}$. Next, we show that $\sqrt{2\zeta_+^2\ln\frac{2}{\omega}}\ge C\cdot\sqrt{\sum_{i=1}^p\zeta_i^2\ln\frac{2d}{\omega}}$, which, by re-arranging the terms, is equivalent to show $\zeta_+^2\ge\frac{C^2}{2}(\sum_{i=i}^p\zeta_i^2)\frac{\log\frac{2d}{\omega}}{\log\frac{2}{\omega}}$. This follows directly from the fact that $\frac{\log\frac{2d}{\omega}}{\log\frac{2}{\omega}}\le 2\log d, \forall\omega\in(0,1)$.
\end{proof}

\begin{lemma}[{\cite[Lemma~C.6]{jin2021nonconvex}}]\label{le-concentration-varying-variance}
    Let random vectors $v_1, \cdots, v_p\in\mathbb{R}^d$, and corresponding filtrations $\mathcal{F}_i=\sigma(v_1, \cdots, v_i)$ for $i\in[k]$ satisfy condition~\ref{con-martingale}, then for any $\iota>0$, and $B>b>0$, there exists an absolute constant $C$ such that, with probability at least $1-2d\log\left(\frac{B}{b}\right)\cdot e^{-\iota}$,
    \begin{equation}
        \sum_{i=1}^p\zeta_i^2\ge B \qquad\text{or}\qquad \left\|\sum_{i=i}^p v_i\right\|\le C\cdot\sqrt{\max\left\{\sum_{i}^p \zeta_i^2, b\right\}\cdot\iota}.
    \end{equation}
\end{lemma}

\begin{lemma}[{\cite[Lemma~C.7]{jin2021nonconvex}}]\label{le-concentration-norm-sum}
    Let random vectors $v_1, \cdots, v_p\in\mathbb{R}^d$, and corresponding filtrations $\mathcal{F}_i=\sigma(v_1, \cdots, v_i)$ for $i\in[k]$ satisfy condition~\ref{con-martingale} with fixed $\zeta_1=\zeta_2=\cdots=\zeta_p=\zeta$, then there exists an absolute constant $C$ such that, for any $\iota>0$, with probability at least $1-e^{-\iota}$,\
    \begin{equation}
        \sum_{i=1}^p\|v_i\|^2\le C\cdot\zeta^2\cdot(p+\iota).
    \end{equation}
\end{lemma}

\begin{lemma}[{Matrix Bernstein inequality~\cite[Theorem~1.4]{tropp2012user}}]\label{le-bernstein}
    Consider a finite sequence $\{\mathbf{M}_i\}_{i\in[k]}$ of independent, random, self-adjoint matrices with dimension $d\times d$. Assume that each random matrix satisfies $\mathbb{E} [\mathbf{M}_i]=\mathbf{0}$, $\|\mathbf{M}_i\|_2\le B $, then for all $t\ge 0$, we have
    \begin{equation}
        \mathbb{P}\left[\left\|\sum_{i\in[k]}\mathbf{M}_i\right\|_2\ge t\right]\le d\exp\left(-\frac{t^2}{2(\sigma^2+Bt/3)}\right),
    \end{equation}
    where $\sigma^2=\left\|\sum_{i\in[k]}\mathbb{E}[\mathbf{M}_i^2]\right\|_2$.
\end{lemma}

\begin{lemma}[{Norm of symmetric matrices with sub-gaussian entries~\cite[Corollary~4.4.8]{vershynin2020high}}]\label{le-norm-symatrix}
    Let $\mathbf{M}$ be an $d\times d$ symmetric random matrix whose entries $\mathbf{M}_{i,j} $ on and above the diagonal are independent, mean zero, sub-gaussian random variables. Then, with probability at least $1-4\exp(-t^2)$, for any $t>0$ we have
    \begin{equation}
        \|\mathbf{M}\|_2\le C\cdot \max_{i,j}\|\mathbf{M}_{i,j}\|_{\psi_2}\cdot (\sqrt{d}+t),
    \end{equation}
    where $C$ is a universal constant.
\end{lemma}

\subsection{Privacy Preliminaries}

\begin{definition}[Gaussian Mechanism~\cite{dwork2014algorithmic}]\label{def-gaussian}
    Given any input data $D\in\mathcal{X}^n$ and a query function $q : \mathcal{X}^n\rightarrow \mathbb{R}^d$, the Gaussian mechanism $\mathcal{M}_G$ is defined as $q(D)+\nu$ where $\nu\sim\mathcal{N}(0,\sigma_G^2\mathbf{I}_d)$. Let $\Delta_2(q)$ be the $\ell_2$-sensitivity of $q$, \emph{i.e.}, $\Delta_2 (q)\coloneqq\sup_{D\sim D^\prime}\|q(D)-q(D^\prime)\|_2$. For any $\sigma, \delta>0$, $\mathcal{M}_G$ guarantees $(\frac{\Delta_2 (q)}{\sigma_G}\sqrt{2\log\frac{1.25}{\delta}}, \delta)$-DP. That is, if we want the output of $q$ to be $(\epsilon, \delta)$-DP for any $0<\epsilon, \delta<1$, then $\sigma_G$ should be set to $\frac{\Delta_2(q)}{\epsilon}\sqrt{2\log\frac{1.25}{\delta}}$.
\end{definition}

\begin{lemma}[Adaptive Composition Theorem~\cite{dwork2014algorithmic}]
    Given target privacy parameters $0<\epsilon <1$ and $0<\delta<1$, to ensure $(\epsilon, \delta)$-DP over $k$-fold adaptive mechanisms, it suffices that each mechanism is $(\epsilon^\prime,\delta^\prime)$-DP, where $\epsilon^\prime=\frac{\epsilon}{2\sqrt{2k\ln(2/\delta)}}$ and $\delta^\prime=\frac{\delta}{2k}$. 
\end{lemma}

\begin{lemma}[Parallel Composition of DP~\cite{mcsherry2009privacy}]
    Suppose there are $n$ $(\epsilon, \delta)$-differentially private mechanisms $\{\mathcal{M}_i\}_{i=1}^n$ and $n$ disjoint datasets denoted by $\{D_i\}_{i=1}^n$. Then the algorithm, which applies each $\mathcal{M}_i$ on the corresponding $D_i$, preserves $(\epsilon, \delta)$-DP in total.
\end{lemma}

\section{Omitted Proofs in Section~\ref{sec-framework}}

\subsection{Proof of Lemma~\ref{le-escape}}~\label{sec-le-escape}

\begin{proof}[Proof of Lemma~\ref{le-escape}]
    We begin by introducing the following notations:
    \begin{align}
        \hat{x}_t &\coloneqq x_t-x_t^\prime, \\
        \hat{\zeta}_t &\coloneqq \zeta_t-\zeta_t^\prime, \\
        \hat{\xi}_t &\coloneqq \xi_t-\xi_t^\prime, \\
        \Delta_t &\coloneqq \int_0^1 \nabla^2F(y\cdot x_t+(1-y)\cdot x_t^\prime)\,\diff y -\mathcal{H}
    \end{align}

    The proof strategy is to derive a contradiction by showing that if the model remains localized (i.e., stays within a radius $\mathcal{R}$ around the saddle point) with high probability, then the coupling sequence must still diverge with non-negligible probability. 
    
    We first characterize the dynamics of $\hat{x}_t$ in the following Lemma~\ref{le-dynamics}. At a high level, we decompose the difference of the coupling sequence $x_t$ into three components: \textbf{(i)} a curvature-dependent term $\mathscr{P}_h(t)$, \textbf{(ii)} a stochastic gradient noise term $\mathscr{P}_{sg}(t)$, \textbf{(iii)} a perturbation-driven term $\mathscr{P}_p(t)$.

\begin{lemma}[Coupling Dynamics]\label{le-dynamics}
    For any $t \ge 0$, the difference between the two coupled iterates satisfies:
    \begin{equation}\label{eq-dynamics}
        \hat{x}_t=-\underbrace{\eta\sum_{i=1}^{t}(\mathbf{I}_d-\eta\mathcal{H})^{t-i}\Delta_{i-1}\hat{x}_{i-1}}_{\mathscr{P}_h(t)}
        -\underbrace{\eta\sum_{i=1}^{t}(\mathbf{I}_d-\eta\mathcal{H})^{t-i}\hat{\zeta}_i}_{\mathscr{P}_{sg}(t)}
        -\underbrace{\eta\sum_{i=1}^{t}(\mathbf{I}_d-\eta\mathcal{H})^{t-i}\hat{\xi}_i }_{\mathscr{P}_p(t)}.
    \end{equation}
\end{lemma}
\begin{proof}[Proof of Lemma~\ref{le-dynamics}]
    By the update rule:
    \begin{align}
        \hat{x}_t
        &=x_t-x_t^\prime\\
        &=\hat{x}_{t-1}-\eta[\nabla F(x_{t-1})-\nabla F(x_{t-1}^\prime)+\zeta_t-\zeta_t^\prime+\xi_t-\xi_t^\prime]\\
        &=\hat{x}_{t-1}-\eta[(\mathcal{H}+\Delta_{t-1})\hat{x}_{t-1}+\hat{\zeta}_t+\hat{\xi}_t]\\
        &=(\mathbf{I}_d-\eta\mathcal{H})\hat{x}_{t-1}-\eta[\Delta_{t-1}\hat{x}_{t-1}+\hat{\zeta}_t+\hat{\xi}_t].
    \end{align}   
    Unrolling the recursion with initial condition $\hat{x}_0 = 0$ yields the desired result:
    \begin{flalign}
        \hat{x}_t
        &=(\mathbf{I}_d-\eta\mathcal{H})^t\hat{x}_0-\eta\sum_{i=1}^t(\mathbf{I}_d-\eta\mathcal{H})^{t-i}(\Delta_{i-1}\hat{x}_{i-1}+\hat{\zeta}_i+\hat{\xi}_i)\\
        &=-\eta\sum_{i=1}^t(\mathbf{I}_d-\eta\mathcal{H})^{t-i}(\Delta_{i-1}\hat{x}_{i-1}+\hat{\zeta}_i+\hat{\xi}_i).
    \end{flalign}
\end{proof}

    Let $\mathcal{E}$ denote the event that both sequences remain localized:
    \[
    \mathcal{E} \coloneqq \left\{\forall t \le \Gamma: \max\left\{\|x_t - \tilde{x}\|, \|x_t^\prime - \tilde{x}\|\right\} \le \mathcal{R} \right\}.
    \]
    We proceed by contradiction. Assume:

    \begin{equation}\label{eq-stuck}
        \mathbb{P}(\mathcal{E}) \ge \frac{3}{4}.
    \end{equation}

    To derive a contradiction, we analyze the terms in~\eqref{eq-dynamics}, showing in Lemma~\ref{le-error1} and Lemma~\ref{le-error2} that the perturbation term $\mathscr{P}_p(t)$ dominates, while the curvature and stochastic gradient terms remain controlled. Define:
    \begin{equation}
        \mathfrak{a}(t) \coloneqq \sqrt{ \sum_{i=1}^t (1 + \eta \gamma)^{2(t - i)} }, \qquad \mathfrak{b}(t) \coloneqq \frac{(1 + \eta \gamma)^t}{\sqrt{2\eta\gamma}}.
    \end{equation}
    It has been verified in \cite[Lemma~29]{jin2021nonconvex} that $\mathfrak{a}(t) \le \mathfrak{b}(t)$ for all $t \in \mathbb{N}$.

\begin{lemma}\label{le-error1}
    For all $t \ge 0$, the following hold:
    \begin{flalign}
        &\mathbb{P}\left[\|\mathscr{P}_p(t)\|\le c\mathfrak{b}(t)\eta r\cdot\sqrt{\iota}\right]\ge 1-2e^{-\iota}\\
        &\mathbb{P}\left[\|\mathscr{P}_p(t)\|\ge\frac{\mathfrak{b}(\Gamma)\eta r}{10}\right]\ge\frac{2}{3}
    \end{flalign}
\end{lemma}

    The proof follows from standard Gaussian concentration and is omitted here; see~\cite[Lemma 30]{jin2021nonconvex}.

\begin{lemma}\label{le-error2}
For all $t \ge 0$, conditioned on $\mathcal{E}$, we have:
    \begin{equation}\label{eq-error2}
        \mathbb{P}\left[\|\mathscr{P}_h(t)+\mathscr{P}_{sg}(t)\|\le\frac{\mathfrak{b}(t)\eta r}{20}\Bigg \vert \mathcal{E}\right]\ge 1-6d\Gamma\log\left(\frac{\mathcal{R}}{\eta r}\right)e^{-\iota}
    \end{equation}
\end{lemma}
    
\begin{proof}[Proof of Lemma~\ref{le-error2}]
    We prove the following strengthened claim for any $t\le\Gamma$ by induction:
    \begin{equation}
        \mathbb{P}\left[\forall i\le t: \|\mathscr{P}_h(i)+\mathscr{P}_{sg}(i)\|\le\frac{\mathfrak{b}(i)\eta r}{20}, \|\mathscr{P}_p(i)\|\le c\mathfrak{b}(i)\eta r\sqrt{\iota}\Bigg\vert \mathcal{E}\right]\le 1-6dt\log\left(\frac{\mathcal{R}}{\eta r}\right)e^{-\iota}.
    \end{equation}
    
    For the base case of $t=0$, the claim holds trivially as $\mathscr{P}_{h}(0)=\mathscr{P}_{sg}(0)=0$. Suppose the claim holds for a step $t<\Gamma$, we then forward prove that the claim also holds for step $t+1\le\Gamma$. 
    Since for $\forall i\le t$, $\|\mathscr{P}_p(i)\|\le c\mathfrak{b}(i)\eta r\sqrt{\iota}$, we have 
    \begin{flalign}
        \|\hat{x}_i\| 
        &\le \|\mathscr{P}_h(i)+\mathscr{P}_{sg}(i)\|+\|\mathscr{P}_{p}(i)\|\\
        &\le \frac{\mathfrak{b}(i)\eta r}{20}+c\mathfrak{b}(i)\eta r\cdot\sqrt{\iota} \\
        &\le 2c\mathfrak{b}(i)\eta r\cdot\sqrt{\iota}.
    \end{flalign}
    Moreover, due to assumption~\eqref{eq-stuck} and the Hessian Lipschitz property, we have 
    \begin{flalign}
        \|\Delta_i\|
        &=\int_0^1 \nabla^2F(y\cdot x_i+(1-y)\cdot x_i^\prime)\,\diff y\\
        &\le \rho\max\{\|x_i-\tilde{x}\|, \|x_i^\prime-\tilde{x}\|\} \le \rho\mathcal{R}.
    \end{flalign}
    With the above upper bounds on $\|\hat{x}_i\|$ and $\|\Delta_i\|$ for $i\le t$, we immediately get for case $t+1$ from the definition of $\mathscr{P}_h(\cdot)$ in \eqref{eq-dynamics} that
    \begin{flalign}
        \|\mathscr{P}_h(t+1)\|
        &\le \eta\rho\mathcal{R}\sum_{i=1}^{t+1}(1+\eta\gamma)^{t+1-i}\left(2c\mathfrak{b}(i)\eta r\sqrt{\iota}\right)\\
        &\le2\eta\rho\mathcal{R}\Gamma c\mathfrak{b}(t+1)\eta r\sqrt{\iota}\le\frac{\mathfrak{b}(t+1)\eta r}{40},
    \end{flalign}
    where the last inequality follows from $2c\eta\rho\mathcal{R}\Gamma=\frac{2c}{s}\le\frac{1}{40}$ for large enough $s$ such that $s\ge80c$.

    Note that $\hat{\zeta}_t|\mathcal{F}_{t-1}\sim\operatorname{nSG}(M\|\hat{x}_t\|)$, by applying Lemma~\ref{le-concentration-varying-variance} with $B=[\mathfrak{a}(t)]^2\eta^2M^2\mathcal{R}^2$ and $b=[\mathfrak{a}(t)]^2\eta^2M^2\eta^2r^2$ therein, we know that, with probability at least $1-4d\log\left(\frac{\mathcal{R}}{\eta r}\right)e^{-\iota}$, we have
    \begin{equation}
        \|\mathscr{P}_{sg}(t+1)\|\le 2c\eta M\sqrt{\Gamma}\mathfrak{b}(t)\eta r\sqrt{\iota}.
    \end{equation}
    For large enough $s$ such that $s\ge(80c)^2$, we have $c\eta M\sqrt{\Gamma\iota}\le\frac{2c}{\sqrt{s}}\le\frac{1}{40}$. Thus,
    \begin{equation}
        \|\mathscr{P}_{sg}(t+1)\|\le c\eta M\sqrt{\Gamma}\mathfrak{b}(t)\eta r\sqrt{\iota}\le\frac{\mathfrak{b}(t)\eta r}{40}.
    \end{equation}
    By Lemma~\ref{le-error1}, we know that, for case $t+1$, with probability at least $1-2 e^{-\iota}$, we have
    \begin{equation}
        \|\mathscr{P}_p(t+1)\|\le c\mathfrak{b}(t+1)\eta r\sqrt{\iota}    
    \end{equation}
    By the union bound, with probability at least $1-\left(6dt\log\left(\frac{\mathcal{R}}{\eta r}\right)e^{-\iota}+4d\log\left(\frac{\mathcal{R}}{\eta r}\right)e^{-\iota}+2e^{-\iota}\right)\ge 1-6d(t+1)\log\left(\frac{\mathcal{R}}{\eta r}\right)e^{-\iota}$,
    \begin{equation}
        \|\mathscr{P}_h(t+1)+\mathscr{P}_{sg}(t+1)\|\le\frac{\mathfrak{b}(t)\eta r}{20}\le\frac{\mathfrak{b}(t+1)\eta r}{20}, \qquad \|\mathscr{P}_p(t+1)\|\le c\mathfrak{b}(t+1)\eta r\sqrt{\iota},
    \end{equation}
    which concludes the proof.
\end{proof}
    
    Now we complete the proof of Lemma~\ref{le-escape}. Choose $\iota$ large enough such that
    \begin{equation}
        \iota \ge \log\left(36 d \Gamma \log\left( \frac{\mathcal{R}}{\eta r} \right)\right),
    \end{equation}
    which is promised by $\mu\ge\frac{1}{s}\log\left(\frac{9d}{C^\frac{1}{4}\eta\sqrt{s\rho\psi}}\log\left(\frac{4C^\frac{1}{4}}{s\eta r}\sqrt{\frac{\psi}{\rho}}\right)\right)$ for sufficiently large numerical constant $s$. Then we have:
    \begin{equation}
        6 d \Gamma \log\left( \frac{\mathcal{R}}{\eta r} \right) e^{-\iota} \le \frac{2}{9}.
    \end{equation}
    From Lemma~\ref{le-error1}, we have:
    \begin{equation}
        \mathbb{P}\left[\|\mathscr{P}_p(\Gamma)\| \ge \frac{\mathfrak{b}(\Gamma) \eta r}{10} \right] \ge \frac{2}{3},
    \end{equation}
    and from Lemma~\ref{le-error2},
    \begin{equation}
        \mathbb{P}\left[\|\mathscr{P}_h(\Gamma) + \mathscr{P}_{sg}(\Gamma)\| \le \frac{\mathfrak{b}(\Gamma) \eta r}{20} \right] \ge \frac{3}{4}\cdot\left(1-6d\Gamma\log\left(\frac{\mathcal{R}}{\eta r}\right)e^{-\iota}\right)\ge\frac{7}{12}
    \end{equation}
    By the union bound, with probability at least $1 - \left(1 - \frac{2}{3}\right) - \left(1 - \frac{7}{12}\right) = \frac{1}{4}$, both events hold:
    \begin{equation}\label{eq-error3}
        \|\mathscr{P}_p(\Gamma)\| \ge \frac{\mathfrak{b}(\Gamma) \eta r}{10}, \quad \|\mathscr{P}_h(\Gamma) + \mathscr{P}_{sg}(\Gamma)\| \le \frac{\mathfrak{b}(\Gamma) \eta r}{20}.
    \end{equation}
    Therefore, using the triangle inequality:
    \begin{flalign}
        &\max\left\{\|x_{\Gamma}-\tilde{x}\|, \|x_{\Gamma}^\prime-\tilde{x}\|\right\}\\
        &\ge\frac{1}{2}\|\hat{x}_\Gamma\|\ge\frac{1}{2}\left[\|\mathscr{P}_p(\Gamma)\|-\|\mathscr{P}_h(\Gamma)+\mathscr{P}_{sg}(\Gamma)\|\right]\ge\frac{\mathfrak{b}(\Gamma)\eta r}{40}=\frac{(1+\eta \gamma)^\Gamma\sqrt{\eta r}}{40\sqrt{2}}\\
        &\ge\frac{(1+\eta\sqrt{\rho\alpha})^{\Gamma}\sqrt{\eta r}}{40\sqrt{2}}\ge\frac{2^{\eta\sqrt{\rho\alpha}\Gamma}\sqrt{\eta r}}{40\sqrt{2}}=\frac{2^{\frac{\iota}{s}}\sqrt{\eta r}}{40\sqrt{2}}=\frac{2^\mu\sqrt{\eta r}}{40\sqrt{2}}>\mathcal{R},
    \end{flalign}
    where the second last inequality is due to the fact $1+a>2^a, \forall a\in(0,1]$ and $\eta\sqrt{\rho\alpha}\le\frac{1}{\iota^2}\le1$, and the last inequality is because $\mu>\log\left(\frac{160\sqrt{2}C^\frac{1}{4}}{s\sqrt{\eta r}}\sqrt{\frac{\psi}{\rho}}\right)$. 
    
    The above means that the localization event $\mathcal{E}$ fails with probability at least $1/4$, i.e., $\mathbb{P}(\mathcal{E})<\frac{3}{4}$, which contradicts with our assumption \eqref{eq-stuck}. Therefore, the assumption \eqref{eq-stuck} should be false, that is, with probability at least $\frac{1}{4}$, $\exists t\le\Gamma, \max\{\|x_t-\tilde{x}\|, \|x_t^\prime-\tilde{x}\|\}\ge\mathcal{R}$, completing the proof.
\end{proof}

\subsection{Proof of Lemma~\ref{le-repeat}}\label{sec-le-repeat}

\begin{proof}[Proof of Lemma~\ref{le-repeat}]
    The failure probability after $Q$ independent repetitions is at most $(7/8)^Q$. Setting $Q = \frac{26}{5} \log (1/\omega_0)$ yields $(7/8)^Q\le\omega_0$, completing the proof.
\end{proof}

\subsection{Proof of Lemma~\ref{le-descent}}\label{sec-le-descent}

\begin{proof}[Proof of Lemma~\ref{le-descent}]
    For any $t \ge 1$, by $M$-smoothness of $F$, we have:   
    \begin{flalign}
        F(x_t)-F(x_{t-1})
        &\le\langle\nabla F(x_{t-1}), x_t-x_{t-1}\rangle+\frac{M}{2}\|x_t-x_{t-1}\|^2\\
        &\le-\eta\langle\nabla F(x_{t-1}), \hat{g}_{t-1}\rangle+\frac{M}{2}\eta^2\|\hat{g}_{t-1}\|^2\\
        &\le-\eta\langle \nabla F(x_{t-1}), \hat{g}_{t-1}\rangle+\frac{\eta}{2}\|\hat{g}_{t-1}\|^2\\
        &\le\frac{\eta}{2}\|\nu_t\|^2-\frac{\eta}{2}\|\nabla F(x_{t-1})\|^2-\frac{\eta}{2}\|\hat{g}_{t-1}\|^2+\frac{\eta}{2}\|\hat{g}_{t-1}\|^2\\
        &=-\frac{\eta}{2}\|\nabla F(x_{t-1})\|^2+\frac{\eta}{2}\|\nu_t\|^2.
    \end{flalign}
    Summing from $t_0 + 1$ to $t_0 + t$, we obtain:
    \begin{equation}
        F(x_{t_0+t})-F(x_{t_0})\le -\frac{\eta}{2}\sum_{i=0}^{t-1}\|\nabla F(x_{t_0+i})\|^2+\frac{\eta}{2}\sum_{i=1}^{t}\|\nu_{t_0+i}\|^2
    \end{equation}
\end{proof}

\subsection{Proof of Corollary~\ref{co-descent}}\label{sec-co-descent}

\begin{proof}[Proof of Corollary~\ref{co-descent}]
    Note that
    \begin{equation}
        \frac{\eta}{2}\sum_{i=1}^{t}\|\nu_{t_0+i}\|^2=\frac{\eta}{2}\sum_{i=1}^{t}\|\zeta_{t_0+i}+\xi_{t_0+i}\|^2\le\eta\sum_{i=1}^{t}(\|\zeta_{t_0+i}\|^2+\|\xi_{t_0+i}\|^2)
    \end{equation}
    By Lemma~\ref{le-concentration-norm-sum}, since $\zeta_i \sim \operatorname{nSG}(\sigma)$, with probability at least $1 - e^{-\iota}$:
    \begin{equation}
        \sum_{i=1}^{t}\|\zeta_{t_0+i}\|^2\le C\cdot\sigma^2(t+\iota).
    \end{equation}
    Using Lemma~\ref{le-gauss-to-normSG}, each $\xi_i \sim \operatorname{nSG}(2\sqrt{2} r \sqrt{d})$, and applying Lemma~\ref{le-concentration-norm-sum} again, with probability at least $1-e^{-\iota}$:
    \begin{equation}
        \sum_{i=1}^{t}\|\xi_{t_0+i}\|^2\le 8C\cdot r^2d(t+\iota).
    \end{equation}
    By the union bound, both bounds hold with probability at least $1 - 2e^{-\iota}$.
\end{proof}

\subsection{Proof of Lemma~\ref{le-decrease}}\label{sec-le-decrease}

\begin{proof}[Proof of Lemma~\ref{le-decrease}]
    We begin with:
    \begin{flalign}
        \|x_{t_0 + \tau} - x_{t_0}\|^2
        &= \eta^2 \left\| \sum_{t = 1}^\tau \nabla F(x_{t_0 + t - 1}) + \nu_{t_0 + t} \right\|^2 \\
        &\le 2 \eta^2 \tau \sum_{t=1}^\tau \left( \|\nabla F(x_{t_0 + t - 1})\|^2 + \|\nu_{t_0 + t}\|^2 \right).
    \end{flalign}
    Following the same argument in the proof of corollary~\ref{co-descent}, with probability at least $1-2e^{-\iota}$,
    \begin{equation}
        \sum_{t=1}^\tau \|\nu_{t_0 + t}\|^2 \le c \cdot \psi^2 (\tau + \iota),
    \end{equation}
    From corollary~\ref{co-descent}, with the same probability of $1-2e^{-\iota}$,
    \begin{equation}
        \sum_{t=1}^{\tau}\|\nabla F(x_{t_0+t-1})\|^2\le\frac{2}{\eta}\left[F(x_{t_0})-F(x_{t_0+\tau})\right]+c\cdot\psi^2(\tau+\iota).
    \end{equation}
    Combining above results, we have, with probability at least $1-2e^{-\iota}$,
    \begin{equation}
        \|x_{t_0+\tau}-x_{t_0}\|^2\le 4\eta \tau[F(x_{t_0})-F(x_{t_0+\tau})]+4c\cdot\eta^2\tau\psi^2(\tau+\iota).
    \end{equation}
    Re-arranging the terms above, we obtain
    \begin{equation}
        F(x_{t_0+\tau})-F(x_{t_0})\le-\frac{1}{4\eta \tau}\|x_{t_0+\tau}-x_{t_0}\|^2+c\cdot\eta\psi^2(\tau+\iota).
    \end{equation}
    According to the criterion for successful escape, we have $\|x_{t_0+\tau}-x_{t_0}\|\ge\mathcal{R}$. Then
    \begin{flalign}
        F(x_{t_0+\tau})-F(x_{t_0})
        &\le-\frac{1}{4\eta \tau}\|x_{t_0+\tau}-x_{t_0}\|^2+c\cdot\eta\psi^2(\tau+\iota)\\
        &\le-\frac{\mathcal{R}^2}{4\eta\Gamma}+c\cdot\eta\psi^2(\Gamma+\iota)\\
        &\le-\frac{s}{4\iota^3}\sqrt{\frac{\alpha^3}{\rho}}+\frac{2c\cdot\psi^2\iota}{s\sqrt{\rho\alpha}}\\
        &\le-\frac{s}{8\iota^3}\sqrt{\frac{\alpha^3}{\rho}}=\Phi,
    \end{flalign}
    where the second to last inequality is from the fact that $s\eta\sqrt{\rho\alpha}=\frac{\rho\alpha}{M^2s\mu^2}<1$, and the last inequality follows from $\alpha\ge 4\sqrt{C}s\mu^2\psi$.
\end{proof}

\subsection{Proof of Lemma~\ref{le-grad-est-error}}\label{sec-le-grad-est-error}

\begin{proof}[Proof of Lemma~\ref{le-grad-est-error}]
    By Corollary~\ref{co-nsg-composition}, for all $t$, $\nu_t \sim \operatorname{nSG}(C \sqrt{\sigma^2 + r^2 d})$. Since $\mathbb{E}[\nu_t] = 0$, by Definition~\ref{def-normsg}, with probability at least $1 - \frac{\omega}{2T}$:
    \begin{equation}
        \|\nu_t\| \le \sqrt{2} C \psi \sqrt{ \log \frac{4T}{\omega} } \le \chi.
    \end{equation}
    Applying a union bound over $t \in [T]$ gives the desired result: with probability at least $1 - \omega/2$, $\|\hat{g}_t - \nabla F(x_{t-1})\| \le \chi$ for all $t$.
\end{proof}

\subsection{Proof of Lemma~\ref{le-steps}}\label{sec-le-steps}

\begin{proof}[Proof of Lemma~\ref{le-steps}]
    By Lemma~\ref{le-grad-est-error}, with probability at least $1 - \omega/2$, the gradient estimation error satisfies $\|\hat{g}_t - \nabla F(x_{t-1})\| \le \chi$ for all $t \in [T]$. We analyze two cases based on whether the algorithm is in the escape phase.
    
    \textbf{Case 1: In escape phase.} 
    When $\|\hat{g}_t\| \le 3\chi$, the escape process is triggered, implying $\|\nabla F(x_{t-1})\| \le \alpha = 4\chi$. The average function decrease per step during a successful escape is at least:
    \begin{equation}
        \frac{\Phi}{\Gamma} = \frac{s^2 \alpha^2 \eta}{8 \iota^4} = \frac{2\chi^2 \eta}{s^2 \mu^4}.
    \end{equation}
    
    \textbf{Case 2: Outside escape phase.}  
    When $\|\hat{g}_t\| > 3\chi$, we have $\|\nabla F(x_{t-1})\| \ge 2\chi$. Each PSGD step yields at least:
    \begin{equation}
        \frac{\eta}{2} (2\chi)^2 = 2\chi^2 \eta > \frac{2\chi^2 \eta}{s^2 \mu^4}.
    \end{equation}
    
    Thus, in either case, the function value decreases by at least $2\chi^2\eta / (s^2\mu^4)$ per step. Denoting $U := F_0 - F^*$, the number of effective descent steps is bounded by:
    \begin{equation}
        T_{\text{effective}} := \frac{U s^2 \mu^4}{2\chi^2 \eta}.
    \end{equation}
    Next, consider the number of $\alpha$-strict saddle points encountered. Each successful escape yields function decrease of at least $\Phi$, so the total number of such escape phases is at most:
    \begin{equation}
        N_{\text{saddle}} := \frac{U}{\Phi} = \frac{8\iota^3 U}{s} \sqrt{\frac{\rho}{\chi^3}}.
    \end{equation}
    By Corollary~\ref{co-escape}, each \texttt{$\Gamma$-descent} succeeds with probability at least $1/8$, and we boost this to $1 - \omega/2$ via the $Q$ independent repetitions in every escape procedure. By Lemma~\ref{le-repeat} with failure probability $\omega_0 = \frac{\omega}{2 N_{\text{saddle}}}$, we require:
    \begin{equation}
        Q = \frac{26}{5} \log \left( \frac{16 \iota^3 U}{s \omega} \sqrt{ \frac{\rho}{\chi^3} } \right).
    \end{equation}
    Hence, the total number of PSGD steps (including all \texttt{$\Gamma$-descent} repetitions) is at most:
    \begin{equation}
        T \le T_{\text{effective}} \cdot Q = \frac{13 U s^2 \mu^4}{5 \chi^2 \eta} \log \left( \frac{16 \iota^3 U}{s \omega} \sqrt{ \frac{\rho}{\chi^3} } \right) = \tilde{O}\left( \frac{U}{\eta \chi^2} \right).
    \end{equation}
\end{proof}

\section{Omitted Proofs in Section~\ref{sec-DP}}

\subsection{Proof of Lemma~\ref{le-adpt-spider-error}}\label{sec-le-adpt-spider-error}

\begin{proof}[Proof of Lemma~\ref{le-adpt-spider-error}]
    Let $\tau(t)$ denote the most recent iteration (up to $t$) at which oracle $\mathcal{O}_1$ was used. 
    
    \textbf{Case 1:} If $t=\tau(t)$, then
    \begin{equation}
        \hat{g}_t = \mathcal{O}_1(x_{t-1}, \mathcal{B}_t) + \xi_t.
    \end{equation}
    Let $\zeta_t := \mathcal{O}_1(x_{t-1}, \mathcal{B}_t) - \nabla F(x_{t-1})$, which is a zero-mean estimator with norm-subGaussian noise due to the $G$-Lipschitz condition:
    \begin{equation}
        \zeta_t \sim \operatorname{nSG}\left(\frac{G\sqrt{\log d}}{\sqrt{b_1}}\right).
    \end{equation}
    The noise term $\xi_t$ is drawn from a Gaussian distribution:
    \begin{equation}
        \xi_t \sim \mathcal{N}\left(0, c_1 \frac{G^2 \log(1/\delta)}{b_1^2 \epsilon^2} \mathbf{I}_d\right).
    \end{equation}
    Thus, in this case, the oracle satisfies condition~\eqref{eq-oracle} with the desired bounds.

    \textbf{Case 2:} If $t > \tau(t)$, then
    \begin{equation}
        \hat{g}_t = \mathcal{O}_1(x_{\tau(t)-1}, \mathcal{B}_{\tau(t)}) + \xi_{\tau(t)} + \sum_{i=\tau(t)+1}^{t} \left( \mathcal{O}_2(x_{i-1}, x_{i-2}, \mathcal{B}_i) + \xi_i \right).
    \end{equation}
    Let $\zeta_{\tau(t)} := \mathcal{O}_1(x_{\tau(t)-1}, \mathcal{B}_{\tau(t)}) - \nabla F(x_{\tau(t)-1})$ and define 
    \begin{equation}
        \zeta^\prime_i := \mathcal{O}_2(x_{i-1}, x_{i-2}, \mathcal{B}_i) - \left(\nabla F(x_{i-1}) - \nabla F(x_{i-2})\right).
    \end{equation}
    Then
    \begin{equation}
        \hat{g}_t - \nabla F(x_{t-1}) = \zeta_{\tau(t)} + \sum_{i=\tau(t)+1}^t \zeta^\prime_i + \xi_{\tau(t)} + \sum_{i=\tau(t)+1}^t \xi_i.
    \end{equation}
    By the $M$-smoothness assumption, we have
    \begin{equation}
        \zeta^\prime_i \sim \operatorname{nSG}\left( \frac{M \|x_{i-1} - x_{i-2}\| \sqrt{\log d}}{\sqrt{b_2}} \right),
    \end{equation}
    and the privacy noise is drawn from
    \begin{equation}
        \xi_i \sim \mathcal{N}\left(0, c_2 \frac{M^2 \log(1/\delta)}{b_2^2 \epsilon^2} \|x_{i-1} - x_{i-2}\|^2 \mathbf{I}_d\right).
    \end{equation}
    Since the algorithm ensures $\textsf{drift}_t := \sum_{i=\tau(t)+1}^t \|x_{i-1} - x_{i-2}\|^2 \le \kappa$, we can bound the noise as follows:
    
    -- From Corollary~\ref{co-nsg-composition}, the total norm-subGaussian parameter becomes:
    \begin{flalign}
        \sigma
        &\le O\left(\sqrt{\left[\left(\frac{G\sqrt{\log d}}{\sqrt{b_1}}\right)^2+\sum_{i=\tau(t)+1}^t\left(\frac{M\|x_{i-1}-x_{i-2}\|\sqrt{\log d}}{\sqrt{b_2}}\right)^2\right]\cdot\log d}\right)\\
        &\le O\left(\sqrt{\frac{G^2\log^2 d}{b_1}+\frac{M^2\log^2 d}{b_2}\kappa}\right).
    \end{flalign}
    -- By the property of Gaussian, the total privacy noise magnitude satisfies:
    \begin{flalign}
        r
        &\le O \left(\sqrt{\frac{G^2\log\frac{1}{\delta}}{b_1^2\epsilon^2}+\sum_{i=\tau(t)+1}^t\left(\frac{M^2\log\frac{1}{\delta}}{b_2^2\epsilon^2}\|x_{t-1}-x_{t-2}\|^2\right)}\right)\\
        &\le O\left(\sqrt{\frac{G^2\log\frac{1}{\delta}}{b_1^2\epsilon^2}+\frac{M^2\log\frac{1}{\delta}}{b_2^2\epsilon^2}\kappa}\right).
    \end{flalign}
\end{proof}

\subsection{Proof of Lemma~\ref{le-driftnum}}\label{sec-le-driftnum}

\begin{proof}[Proof of Lemma~\ref{le-driftnum}]
    By the $M$-smoothness assumption and using the fact $\eta \le \frac{1}{M}$, we apply the standard descent lemma:
    \begin{flalign*}
        F(x_t)-F(x_{t-1})
        &\le \langle\nabla F(x_{t-1}), x_t-x_{t-1} \rangle+\frac{M}{2}\|x_t-x_{t-1}\|^2\\
        &\le \langle\nabla F(x_{t-1})-\hat{g}_t, -\eta \cdot\hat{g}_t \rangle -\eta\|\hat{g}_t\|^2 +\frac{\eta}{2}\|\hat{g}_t\|^2\\
        &\le \eta\|\nabla F(x_{t-1})-\hat{g}_t\|\|\hat{g}_t\|_2-\frac{\eta}{2}\|\hat{g}_t\|^2.
    \end{flalign*}
    By Lemma~\ref{le-grad-est-error}, with probability at least $1 - \omega/2$, we have $\|\nabla F(x_{t-1}) - \hat{g}_t\| \le \chi$ for all $t$.

    Now consider two cases:

    \textbf{Case 1:} If $\|\nabla F(x_{t-1})\| \ge 4\chi$, then
    \begin{equation}
        \|\hat{g}_t\| \ge \|\nabla F(x_{t-1})\| - \chi \ge 3\chi \ge 3 \|\nabla F(x_{t-1}) - \hat{g}_t\|,
    \end{equation}
    yielding
    \begin{equation}
        F(x_t) - F(x_{t-1}) \le -\frac{\eta}{6} \|\hat{g}_t\|^2.
    \end{equation}
    
    \textbf{Case 2:} If $\|\nabla F(x_{t-1})\| \le 4\chi$, then $\|\hat{g}_t\| \le 5\chi$, and thus
    \begin{equation}
        F(x_t) - F(x_{t-1}) \le 5 \eta \chi^2.
    \end{equation}
    Let $\mathcal{T} = \{t_1, t_2, \dots, t_{|\mathcal{T}|}\}$ denote the set of iterations where model drift exceeds $\kappa$. For each pair of successive drift resets:
    \begin{flalign}
        F(x_{t_{i+1}})-F(x_{t_i})
        & \le -\frac{1}{6\eta}\sum_{t=t_i+1}^{t_{i+1}}\eta^2\|\hat{g}_t\|_2^2+(t_{i+1}-t_i)5\eta\chi^2\\
        & \le -\frac{1}{6\eta}\operatorname{drift}_{t_{i+1}}+(t_{i+1}-t_i)5\eta\chi^2
        \le -\frac{1}{6\eta}\kappa+(t_{i+1}-t_i)5\eta\chi^2.
    \end{flalign}
    Summing over $i$, we obtain:
    \begin{equation*}
        F(x_{t_{|\mathcal{T}|}})-F(x_{t_1})\le -\frac{|\mathcal{T}|}{6\eta}\kappa+5T\eta\chi^2.
    \end{equation*}
    Since $F(\cdot)$ is upper bounded by $U$, we must have:
    \begin{equation}
        - U \le -\frac{|\mathcal{T}| \kappa}{6\eta} + 5T \eta \chi^2,
    \end{equation}
    which yields:
    \begin{equation*}
        |\mathcal{T}|\le O\left(\frac{U\eta}{\kappa}+\frac{T\eta^2 \chi^2}{\kappa}\right)=O\left(\frac{U\eta}{\kappa}\right),
    \end{equation*}
    using $T = O(U / (\eta \chi^2))$.
\end{proof}

\subsection{Proof of Theorem~\ref{th-adapt-spider}}\label{sec-th-adapt-spider}

\begin{proof}[Proof of Theorem~\ref{th-adapt-spider}]
    We first verify that the batch size settings $b_1$ and $b_2$ are feasible, i.e., the total number of data samples used remains $O(n)$. Recall from Lemma~\ref{le-driftnum} that the number of rounds where drift exceeds the threshold is bounded by $|\mathcal{T}| = O(U\eta/\kappa)$, and the total number of steps is $T = O(U/(\eta\chi^2))$. Then:
    \begin{equation}
        b_1 \cdot |\mathcal{T}| + b_2 \cdot (T - |\mathcal{T}|) \le b_1 \cdot |\mathcal{T}| + b_2 \cdot T \le O(n),
    \end{equation}
    under our settings of $b_1 = \frac{n\kappa}{2U\eta}$ and $b_2 = \frac{n\eta\chi^2}{2U}$. This confirms feasibility.

    Since each sample is used only once, the overall $(\epsilon, \delta)$-differential privacy guarantee follows directly from the Gaussian mechanism and the parallel composition theorem.
    
    We now derive the convergence error $\alpha$ via Theorem~\ref{th-framework}, which gives:
    \begin{equation}
        \alpha = O(\chi) = \tilde{O}(\psi) = \tilde{O}(\sqrt{\sigma^2 + r^2 d}),
    \end{equation}
    where from Lemma~\ref{le-adpt-spider-error}:
    \begin{equation}
        \sigma^2 \le \tilde{O}\left(\frac{G^2}{b_1} + \frac{M^2 \kappa}{b_2} \right), \quad
        r^2 \le \tilde{O}\left(\frac{G^2}{b_1^2 \epsilon^2} + \frac{M^2 \kappa}{b_2^2 \epsilon^2} \right).
    \end{equation}
    Substituting our settings $b_1 = \frac{n\kappa}{2U\eta}$ and $b_2 = \frac{n\eta\chi^2}{2U}$ into the expression, we get:
    \begin{flalign}
        \alpha &= \tilde{O}\left( \sqrt{ \frac{G^2 U \eta}{n \kappa} 
        + \frac{G^2 d U^2 \eta^2}{n^2 \epsilon^2 \kappa^2} 
        + \frac{M^2 U \kappa}{n \eta \chi^2} 
        + \frac{M^2 d U^2 \kappa}{n^2 \epsilon^2 \eta^2 \chi^4} } \right) \\
        &= \tilde{O}\left( \sqrt{ \frac{G^2 U \sqrt{\rho \alpha}}{M^2 n \kappa} 
        + \frac{G^2 d U^2 \rho \alpha}{M^4 n^2 \epsilon^2 \kappa^2}
        + \frac{M^4 U \kappa}{\sqrt{\rho} n \alpha^{5/2}}
        + \frac{M^6 d U^2 \kappa}{\rho n^2 \epsilon^2 \alpha^5} } \right).
    \end{flalign}
    To isolate $\alpha$, we take the largest among the resulting bounds:
    \begin{equation}
        \alpha = \tilde{O}\left( \max\left\{
        \left( \frac{G^2 U \sqrt{\rho}}{M^2 n \kappa} \right)^{2/3},
        \frac{G^2 d U^2 \rho}{M^4 n^2 \epsilon^2 \kappa^2},
        \left( \frac{M^4 U \kappa}{n \sqrt{\rho}} \right)^{2/9},
        \left( \frac{M^6 d U^2 \kappa}{\rho n^2 \epsilon^2} \right)^{1/7}
        \right\} \right).
    \end{equation}
    Now set:
    \begin{equation}
        \kappa = \max\left\{
        \frac{G^{3/2} U^{1/2} \rho^{1/2}}{M^{5/2} n^{1/2}},
        \frac{G^{14/15} d^{2/5} U^{4/5} \rho^{8/15}}{M^{34/15} (n\epsilon)^{4/5}}
        \right\}.
    \end{equation}
    Substituting this into the above expression of $\alpha$ yields:
    \begin{equation}
        \alpha = \tilde{O}\left(
        \left( \frac{G U M}{n} \right)^{1/3}
        + \frac{G^{2/15} U^{2/5} M^{8/15}}{\rho^{1/15}} \left( \frac{\sqrt{d}}{n\epsilon} \right)^{2/5}
        \right) = \tilde{O}\left( \frac{1}{n^{1/3}} + \left( \frac{\sqrt{d}}{n\epsilon} \right)^{2/5} \right).
    \end{equation}
\end{proof}

\section{Omitted Proofs in Section~\ref{sec-distributed}}

\subsection{Proof of Lemma~\ref{le-dis-adpt-spider-error}}\label{sec-le-dis-adpt-spider-error}

\begin{proof}[Proof of Lemma~\ref{le-dis-adpt-spider-error}]
    Let $\tau(t)$ denote the most recent iteration at which oracle $\mathcal{O}_1$ was queried before or at iteration $t$.

    \textbf{Case 1:} If $t=\tau(t)$, then the global estimator is
    \begin{equation}
        \hat{g}_t = \frac{1}{m} \sum_{j=1}^m \left( \mathcal{O}_1(x_{t-1}, \mathcal{B}_{j,t}) + \xi_{j,t} \right).
    \end{equation}
    Each $\mathcal{O}_1(x_{t-1}, \mathcal{B}_{j,t})$ is an unbiased estimate of $\nabla F_j(x_{t-1})$. Let $\zeta_{j,t} := \mathcal{O}_1(x_{t-1}, \mathcal{B}_{j,t}) - \nabla F_j(x_{t-1})$, and define $\zeta_t := \frac{1}{m} \sum_j \zeta_{j,t}$ and $\xi_t := \frac{1}{m} \sum_j \xi_{j,t}$. Then, 
    \begin{equation}
        \hat{g}_t - \nabla F(x_{t-1}) = \zeta_t + \xi_t.
    \end{equation}

    Since $f$ is $G$-Lipschitz, we have $\zeta_t \sim \operatorname{nSG}\left( \frac{G \sqrt{\log d}}{\sqrt{m b_1}} \right)$. Each $\xi_{j,t} \sim \mathcal{N}\left( 0, c_1 \frac{G^2 \log(1/\delta)}{b_1^2 \epsilon^2} \mathbf{I}_d \right)$, so their average satisfies:
    \begin{equation}
        \xi_t \sim \mathcal{N}\left( 0, c_1 \frac{G^2 \log(1/\delta)}{m b_1^2 \epsilon^2} \mathbf{I}_d \right).
    \end{equation}
    Thus, in this case, the oracle satisfies condition~\eqref{eq-oracle} with the desired bounds.

    \textbf{Case 2:} If $t > \tau(t)$, the global estimate is:
    \begin{equation}
        \hat{g}_t = \frac{1}{m} \sum_{j=1}^m \left( \mathcal{O}_1(x_{\tau(t)-1}, \mathcal{B}_{j,\tau(t)}) + \xi_{j,\tau(t)} + \sum_{i=\tau(t)+1}^{t} \left[ \mathcal{O}_2(x_{i-1}, x_{i-2}, \mathcal{B}_{j,i}) + \xi_{j,i} \right] \right).
    \end{equation}
    Let $\zeta_{j,\tau} := \mathcal{O}_1(x_{\tau(t)-1}, \mathcal{B}_{j,\tau(t)}) - \nabla F_j(x_{\tau(t)-1})$, and define:
    \begin{equation}
        \zeta^\prime_{j,i} := \mathcal{O}_2(x_{i-1}, x_{i-2}, \mathcal{B}_{j,i}) - \left[ \nabla F_j(x_{i-1}) - \nabla F_j(x_{i-2}) \right].
    \end{equation}
    Then,
    \begin{equation}
        \hat{g}_t - \nabla F(x_{t-1}) = \zeta_{\tau(t)} + \sum_{i=\tau(t)+1}^t \zeta_i^\prime + \xi_{\tau(t)} + \sum_{i=\tau(t)+1}^t \xi_i,
    \end{equation}
    where $\zeta_{\tau(t)} := \frac{1}{m} \sum_j \zeta_{j,\tau(t)}$, $\zeta_i^\prime := \frac{1}{m} \sum_j \zeta_{j,i}^\prime$, and similarly for $\xi_{\tau(t)}$ and $\xi_i$. By the $M$-smoothness of $f$, we have:
    \begin{equation}
        \zeta_i^\prime \sim \operatorname{nSG}\left( \frac{M \|x_{i-1} - x_{i-2}\| \sqrt{\log d}}{\sqrt{m b_2}} \right),
        \quad
        \xi_i \sim \mathcal{N}\left( 0, c_2 \frac{M^2 \log(1/\delta)}{m b_2^2 \epsilon^2} \|x_{i-1} - x_{i-2}\|^2 \mathbf{I}_d \right).
    \end{equation}
    Since the algorithm ensures that $\textsf{drift}_t := \sum_{i=\tau(t)+1}^t \|x_{i-1} - x_{i-2}\|^2 \le \kappa$, we obtain:
    \begin{equation}
        \sigma = \tilde{O}\left( \sqrt{ \frac{G^2 \log^2 d}{m b_1} + \frac{M^2 \log^2 d}{m b_2} \kappa } \right),
        \quad
        r = \tilde{O}\left( \sqrt{ \frac{G^2 \log(1/\delta)}{m b_1^2 \epsilon^2} + \frac{M^2 \log(1/\delta)}{m b_2^2 \epsilon^2} \kappa } \right).
    \end{equation}
\end{proof}

\subsection{Proof of Theorem~\ref{th-dis-adapt-spider}}\label{sec-th-dis-adapt-spider}

\begin{proof}[Proof of Theorem~\ref{th-dis-adapt-spider}]
    We first verify that the total sample usage per client is $O(n)$. From Lemma~\ref{le-driftnum}, we have $|\mathcal{T}| = O(U \eta / \kappa)$ and $T = O(U / (\eta \chi^2))$. Using the settings:
    \begin{equation}
        b_1 = \frac{n \kappa}{2U \eta}, \quad b_2 = \frac{n \eta \chi^2}{2U},
    \end{equation}
    the total number of samples used per client is:
    \begin{equation}
        b_1 \cdot |\mathcal{T}| + b_2 \cdot (T - |\mathcal{T}|) \le b_1 \cdot |\mathcal{T}| + b_2 \cdot T = O(n).
    \end{equation}

    Differential privacy guarantees follows from the Gaussian mechanism and parallel composition, since each data point is used at most once.
    
    Now for the error analysis. By Theorem~\ref{th-framework}:
    \begin{equation}
        \alpha = O(\chi) = \tilde{O}(\psi) = \tilde{O}(\sqrt{\sigma^2 + r^2 d}).
    \end{equation}
    From Lemma~\ref{le-dis-adpt-spider-error}:
    \begin{equation}
        \alpha = \tilde{O}\left( \sqrt{ \frac{G^2}{m b_1} + \frac{G^2 d}{m b_1^2 \epsilon^2} + \left( \frac{M^2}{m b_2} + \frac{M^2 d}{m b_2^2 \epsilon^2} \right) \cdot \kappa } \right).
    \end{equation}
    Substitute the expressions for $b_1$, $b_2$ into the bound and simplify, we get:
    \begin{flalign}
        \alpha&=\tilde{O}\left(\sqrt{\frac{G^2U\eta}{m n\kappa}+\frac{G^2dU^2\eta^2}{m n^2\epsilon^2\kappa^2}+\frac{M^2U\kappa}{m n\eta\chi^2}+\frac{M^2dU^2\kappa}{m n^2\epsilon^2\eta^2\chi^4}}\right)\\
        &=\tilde{O}\left(\sqrt{\frac{G^2U\sqrt{\rho\alpha}}{m M^2n\kappa}+\frac{G^2dU^2\rho\alpha}{m n^2\epsilon^2M^4\kappa^2}+\frac{M^4U\kappa}{mn\rho^\frac{1}{2}\alpha^\frac{5}{2}}+\frac{M^6dU^2\kappa}{mn^2\epsilon^2\rho\alpha^5}}\right).
    \end{flalign}
    To isolate $\alpha$, we take the largest among the resulting bounds:
    \begin{flalign*}
        \alpha 
        = & \tilde{O}\left( \max \left\{ 
        \left(\frac{G^2 U\sqrt{\rho}}{mM^2n\kappa}\right)^{2/3}, \frac{G^2dU^2\rho}{mn^2\epsilon^2M^4\kappa^2}, \left(\frac{M^4U\kappa}{mn\sqrt{\rho}}\right)^{2/9},\left(\frac{M^6dU^2\kappa}{m  n^2\epsilon^2\rho}\right)^{1/7}\right\}\right).
    \end{flalign*}
    Now set:
    \begin{equation}
        \kappa=\max\left\{\frac{G^{3/2}\sqrt{\rho U}}{M^{5/2}\sqrt{mn}}, \frac{G^{14/15}d^{2/5}U^{4/5}\rho^{8/15}}{M^{34/15}(\sqrt{m}n\epsilon)^{4/5}}\right\}
    \end{equation}
    Substituting this into the above expression of $\alpha$ yields:
    \begin{equation}
        \alpha
        = \tilde{O}\left(\left(\frac{GUM}{mn}\right)^{1/3}+\frac{G^{2/15}U^{2/5}M^{8/15}}{\rho^{1/15}}\left(\frac{\sqrt{d}}{\sqrt{m}n\epsilon}\right)^{2/5}\right)
        =\tilde{O}\left(\frac{1}{(mn)^{1/3}}+\left(\frac{\sqrt{d}}{\sqrt{m}n\epsilon}\right)^{2/5}\right).
    \end{equation}
\end{proof}

\subsection{Proof of Theorem~\ref{th-select}}\label{sec-th-select}

\begin{proof}[Proof of Theorem~\ref{th-select}]
    The $(\epsilon, \delta)$-ICRL-DP guarantee follows directly from the Gaussian mechanism and the adaptive composition theorem, since each client adds independent Gaussian noise to both their gradient and Hessian estimates. Each local data point is used at most $T$ times—once for each model iterate—and all messages sent to the server are privatized accordingly.

    We now derive the error rate $\alpha$ guarantee for the output $x_o$.
    Let $\mathcal{S} \coloneqq \bigsqcup_{j=1}^m S_j$ denote the full held-out evaluation dataset, and let $x_p$ be an $\alpha$-SOSP in the input to Algorithm~\ref{algo-select}. Define the aggregate gradient noise and Hessian noise as
    \begin{equation}
        \theta_p \coloneqq \frac{1}{m} \sum_{j=1}^m \theta_{j,p}, \quad \mathbf{H}_p \coloneqq \frac{1}{m} \sum_{j=1}^m \mathbf{H}_{j,p}.
    \end{equation}
    Let $\sigma_1^2 = c_1 \frac{G^2 T \log(1/\delta)}{n^2 \epsilon^2}$ and $\sigma_2^2 = c_2 \frac{M^2 d T \log(1/\delta)}{n^2 \epsilon^2}$ denote the variances of the noise added to the gradient and Hessian components, respectively.

    \textbf{Gradient Estimation Error.} For any $\mathcal{S}_j$ and $x$, $\nabla \hat{f}_{\mathcal{S}_j}(x)-\nabla F_j(x)$ is zero-mean and follows $\operatorname{nSG}\left(\frac{2G}{\sqrt{n}}\right)$. By the $G$-Lipschitz assumption and norm-sub-Gaussian concentration (Lemma~\ref{le-hoeffding}), we have with probability at least $1 - \omega'/8$:
    \begin{equation}
        \|\nabla F(x_p) - \nabla \hat{f}_{\mathcal{S}}(x_p)\| \le O\left( \frac{G \sqrt{\log(d/\omega')}}{\sqrt{mn}} \right).
    \end{equation}
    Also, since $\theta_p \sim \mathcal{N}(0, \sigma_1^2/m)$, standard Gaussian concentration (Lemma~\ref{le-gauss-to-normSG}) gives, with probability at least $1 - \omega'/8$:
    \begin{equation}
         \|\theta_p\| \le O\left( \frac{G \sqrt{d T \log(1/\delta) \log(1/\omega')}}{\sqrt{m} n \epsilon} \right).
    \end{equation}

    \textbf{Hessian Estimation Error.} For any $j\in[m]$ and $z\in\mathcal{S}_j$, $\mathbb{E}[\nabla^2 f(x_p; z)-\nabla^2 F_j(x_p)]=0$, and $\|\nabla^2 f(x_p;z)-\nabla^2 F_j(x_p)\|_2\le 2M$ (due to $M$-smoothness). That is, each empirical Hessian term is $2M$-bounded in operator norm. Applying the matrix Bernstein inequality (Lemma~\ref{le-bernstein}), and using the assumption $mn \ge \frac{4}{9} \log(8d/\omega')$, we obtain with probability at least $1 - \omega'/8$:
    \begin{equation}
        \left\| \nabla^2 \hat{f}_{\mathcal{S}}(x_p) - \nabla^2 F(x_p) \right\| \le O\left( M \sqrt{ \frac{ \log(d/\omega') }{ mn } } \right).
    \end{equation}
    For the added noise, since $\mathbf{H}_p$ consists of symmetric Gaussian matrices with variance $\sigma_2^2 / m$, Lemma~\ref{le-norm-symatrix} gives, with probability at least $1 - \omega'/8$:
    \begin{equation}
        \|\mathbf{H}_p\| \le O\left( \frac{M d \sqrt{T \log(1/\delta) \log(1/\omega')}}{\sqrt{m} n \epsilon} \right).
    \end{equation}

    \textbf{Verification for $x_p$.} Combining the above estimates and using a union bound, with probability at least $1 - \omega'/2$, we have:
    \begin{flalign}
        \|\nabla \bar{F} (x_p)\|_2
        & \le \|\nabla F (x_p)\|_2 + \|\nabla \bar{F} (x_p)-\nabla F(x_p)\|_2\\
        & \le \|\nabla F (x_p)\|_2 + \|\nabla \hat{f}_{\mathcal{S}}(x_p)-\nabla F(x_p)\|_2 + \|\theta_p\|_2 \\
        & \le \alpha + \text{(estimation error)}\\
        & \le O\left( \alpha + \frac{G\log\left(d/\omega^\prime\right)}{\sqrt{mn}}+ \frac{G\sqrt{dT\log\left(1/\delta\right)\log\left(1/\omega^\prime\right)}}{\sqrt{m}n\epsilon} \right),
    \end{flalign}
    and
    \begin{flalign}
        \lambda_{\min}\left(\nabla^2 \bar{F} (x_p)\right) 
        & \ge \lambda_{\min}\left(\nabla^2 F (x_p)\right) + \lambda_{\min}\left(\nabla^2 \bar{F} (x_p)-\nabla^2 F(x_p)\right) \\
        & \ge \lambda_{\min}\left(\nabla^2 F (x_p)\right) + \lambda_{\min}\left(\nabla^2 \hat{f}_{\mathcal{S}} (x_p)-\nabla^2 F(x_p)\right) + \lambda_{\min}\left(\mathbf{H}_p\right) \\
        & \ge -\sqrt{\rho\alpha} -\left\|\nabla^2 f(x_p; \mathcal{S})-\nabla^2 F(x_p)\right\|_2 -\|\mathbf{H}_p\|_2 \\
        & \ge -\left( \sqrt{\rho \alpha} + \text{(estimation error)} \right)\\
        & \ge -O\left(\sqrt{\rho \alpha} + M\sqrt{\frac{\log\left(d/\omega^\prime\right)}{mn}}+\frac{M d\sqrt{T\log(1/\delta)\log\left(1/\omega^\prime\right)}}{\sqrt{m}n\epsilon}\right).
    \end{flalign}
    Hence, $x_p$ will be selected with probability at least $1 - \omega^\prime/2$.

    \textbf{Guarantee for Output \( x_o \).} Let $x_o$ be the output of Algorithm~\ref{algo-select}. By construction, it must satisfy:
    \begin{flalign}
        \|\nabla F(x_o)\|_2
        & \le \|\nabla \bar{F}(x_o)\|_2 + \|\nabla F(x_o)-\nabla \bar{F}(x_o)\|_2 \\
        & \le \|\nabla \bar{F}(x_o)\|_2 + \|\nabla F(x_o)-\nabla \hat{f}_{\mathcal{S}}(x_o)\|_2+\|\xi_{o}\|_2,
    \end{flalign}
    and
    \begin{flalign}
        \lambda_{\min}(\nabla^2 F(x_o))
        &\ge \lambda_{\min}(\nabla^2 \bar{F}(x_o)) + \lambda_{\min}(\nabla^2 F(x_o)-\nabla^2 \bar{F}(x_o))\\
        &\ge \lambda_{\min}(\nabla^2 \bar{F}(x_o)) - \|\nabla^2 F(x_o)-\nabla^2 \bar{F}(x_o)\|_2 \\
        &\ge \lambda_{\min}(\nabla^2 \bar{F}(x_o)) - \|\nabla^2 F(x_o)-\nabla^2 \hat{f}_{\mathcal{S}}(x_o)\|_2 -\|H_o\|_2.
    \end{flalign}

    Using the same reasoning as above, applying the union bound again and using the fact that $x_o$ is the output, we get that with probability at least $1 - \omega'$, the following hold:
    \begin{equation}
        \|\nabla F(x_o)\| \le O\left( \alpha + \frac{G \log(d/\omega')}{\sqrt{mn}} + \frac{G \sqrt{d T \log(1/\delta) \log(1/\omega')}}{\sqrt{m} n \epsilon} \right),
    \end{equation}
    and
    \begin{equation}
        \lambda_{\min}(\nabla^2 F(x_o)) \ge -O\left( \sqrt{\rho \alpha} + M \sqrt{ \frac{ \log(d/\omega') }{ mn } } + \frac{ M d \sqrt{ T \log(1/\delta) \log(1/\omega') } }{ \sqrt{m} n \epsilon } \right).
    \end{equation}
    Finally, recalling that $T = O(1/\alpha^{2.5})$, and grouping the dependency on $\alpha$, $d$, $m$, $n$, and $\epsilon$, we conclude that $x_o$ is an $\alpha'$-SOSP with
    \begin{equation}
        \alpha' = \tilde{O}\left( \alpha + \frac{1}{mn} + \frac{1}{\sqrt{mn}} + \frac{\alpha}{\sqrt{mn}} + \frac{\sqrt{d}}{\sqrt{m} n \epsilon \alpha^{5/4}} + \frac{d}{\sqrt{m} n \epsilon \alpha^{3/4}} + \frac{d^2}{m n^2 \epsilon^2 \alpha^{5/2}} \right),
    \end{equation}
    as claimed.
\end{proof}

\section{Experiments}\label{sec-exp}

\paragraph{Running Environments}
All experiments were conducted with the following computing infrastructure:  
\begin{itemize}
    \item OS: Ubuntu 22.04.4 LTS
    \item CPU: AMD EPYC 7513 32-Core Processor
    \item CPU Memory: 503GB
    \item GPU: NVIDIA RTX A6000 GPU
    \item GPU Memory: 48GB
    \item Programming language: Python 3.11.8
    \item Deep learning framework: Pytorch 2.2.2 + cuda 12.1
\end{itemize}

\begin{figure*}[htbp]
\centering
{
  \subfigure[Test Accuracy ($\epsilon = 0.5$)]{\includegraphics[width=0.32\textwidth]{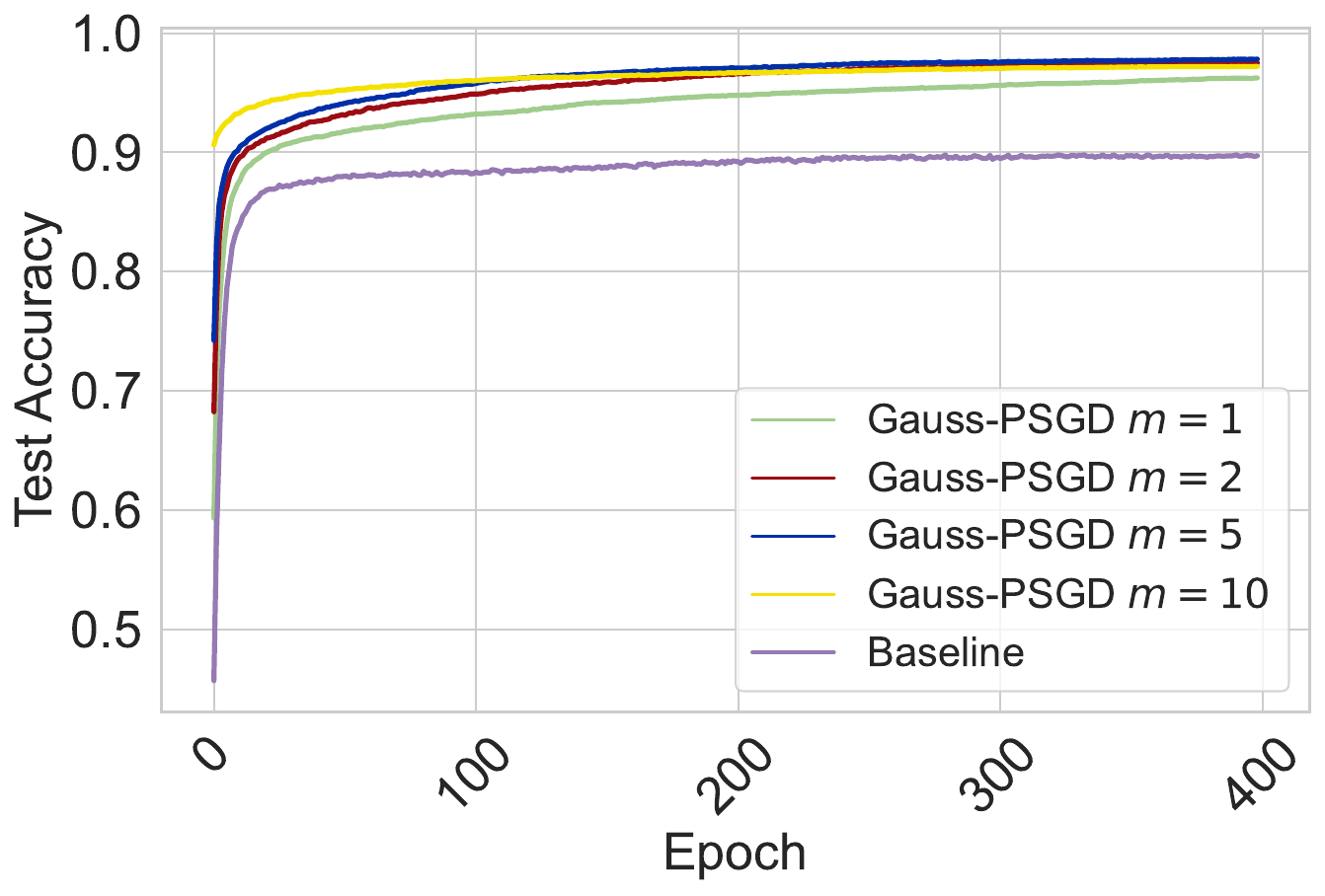}
  }\hfill
  \subfigure[Test Accuracy ($\epsilon = 1.0$)]{\includegraphics[width=0.32\textwidth]{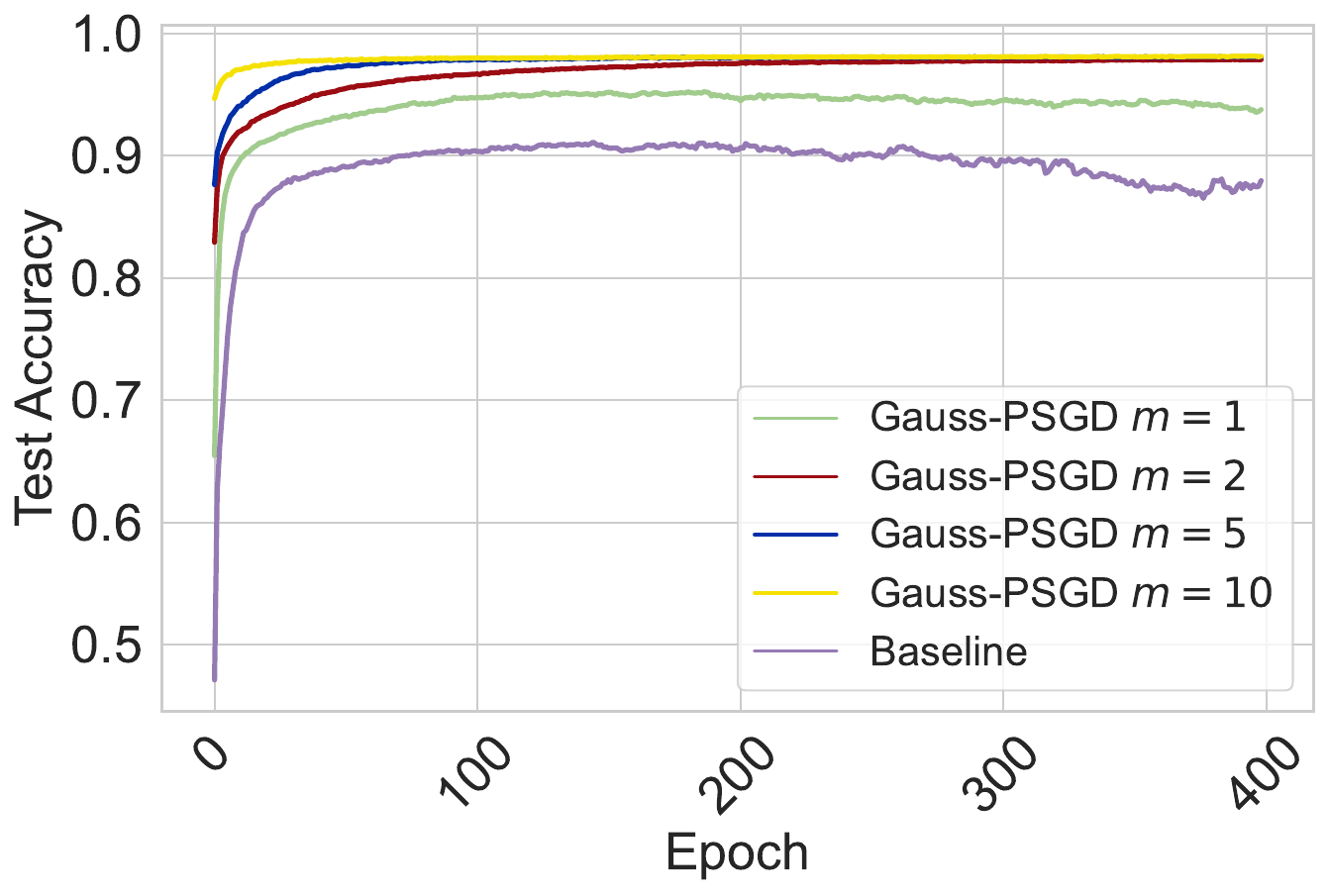}
  }\hfill
  \subfigure[Test Accuracy ($\epsilon = 2.0$)]{\includegraphics[width=0.32\textwidth]{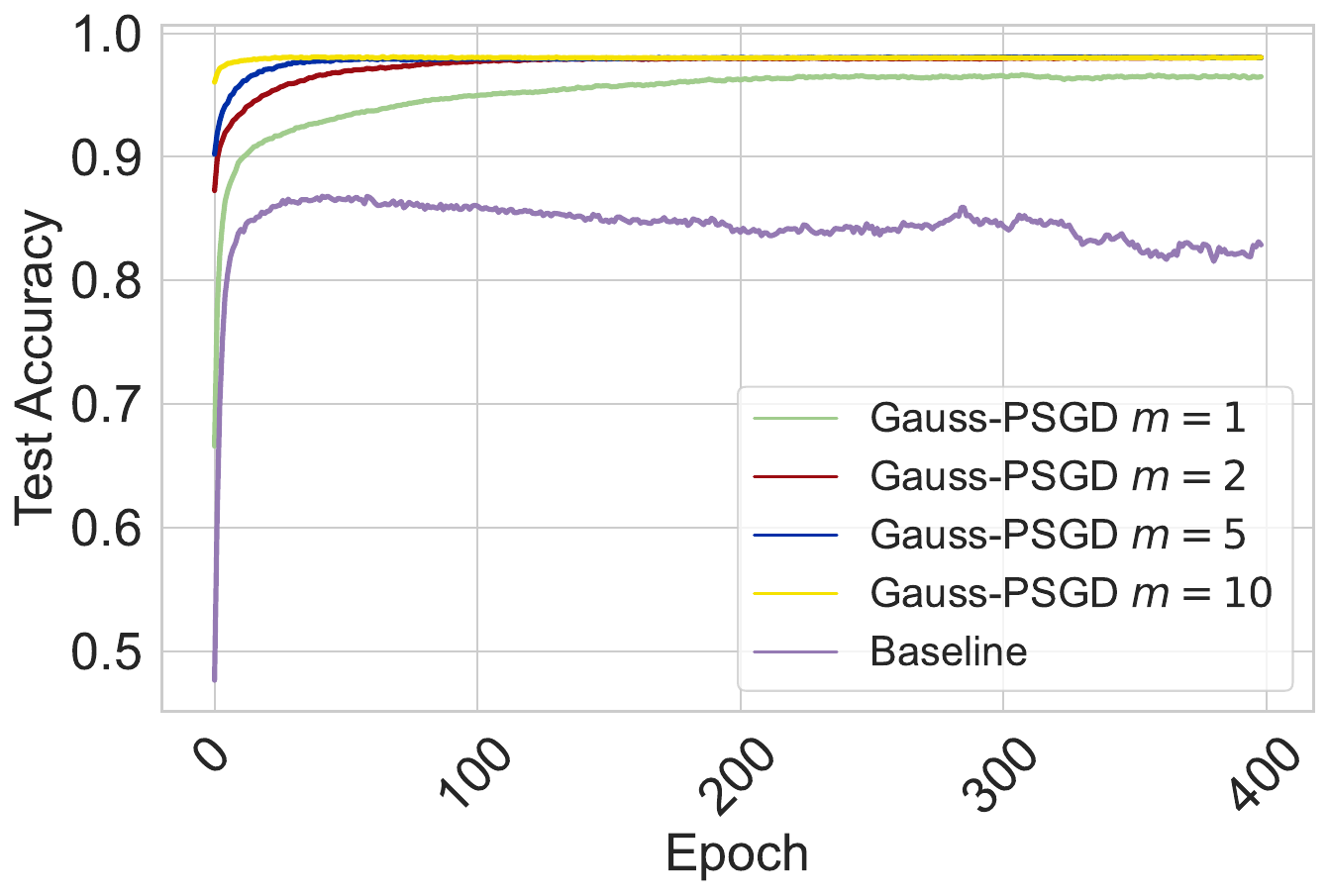}
  }\hfill
  \\
  \subfigure[Test Loss ($\epsilon = 0.5$)]{\includegraphics[width=0.32\textwidth]{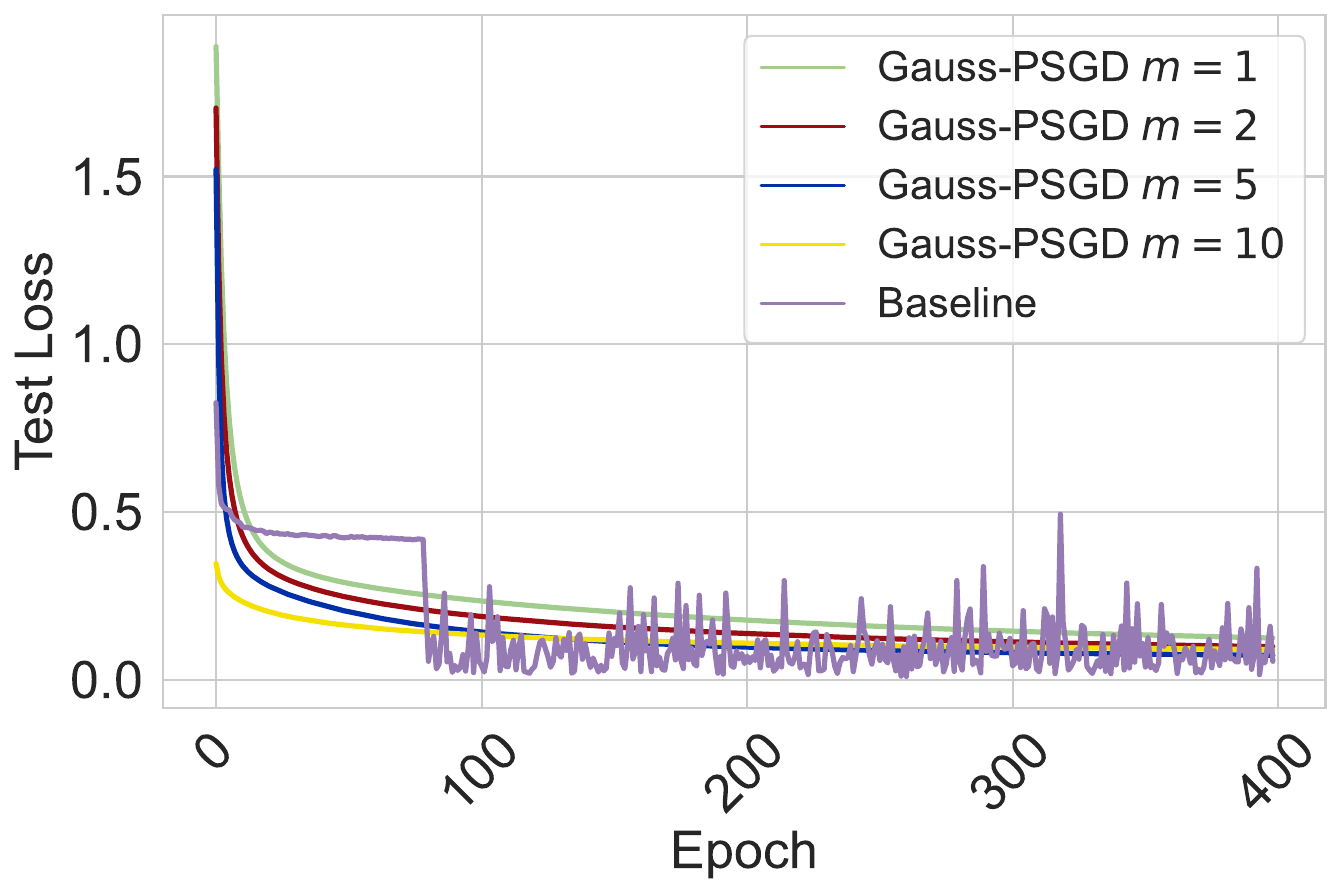}
  }\hfill
  \subfigure[Test Loss ($\epsilon = 1.0$)]{\includegraphics[width=0.32\textwidth]{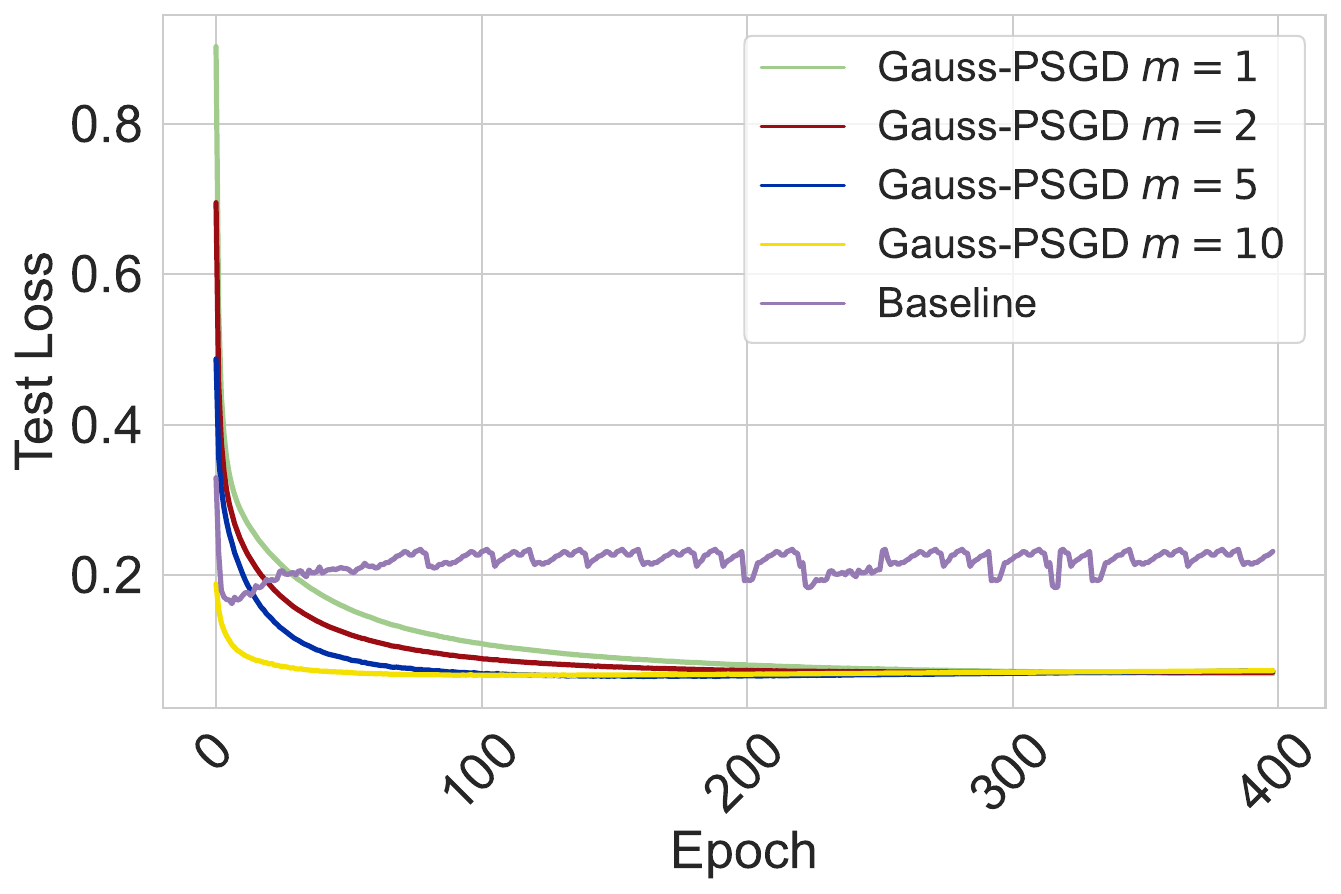}
  }\hfill
  \subfigure[Test Loss ($\epsilon = 2.0$)]{\includegraphics[width=0.32\textwidth]{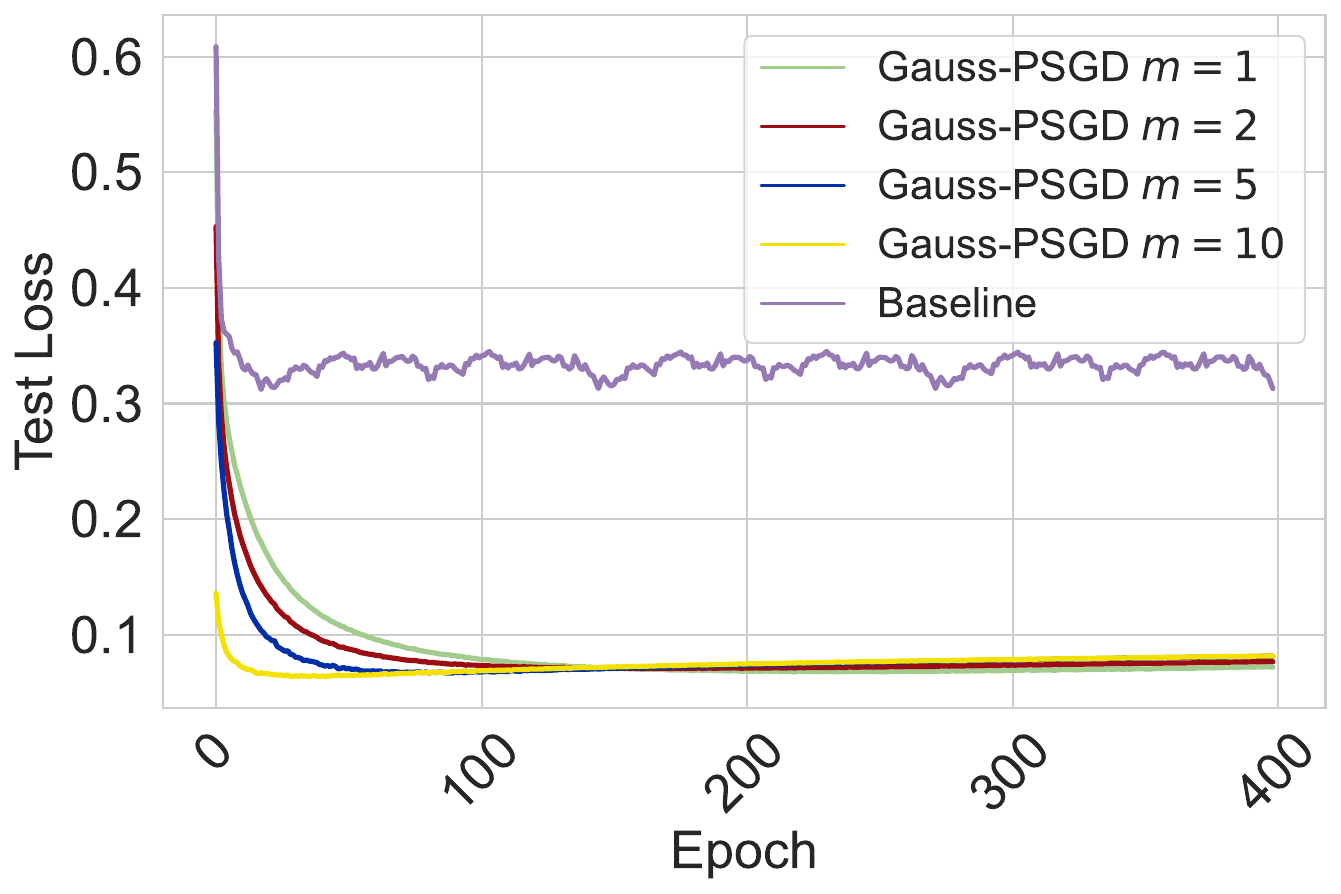}
  }\hfill
   \caption{Comparison of learning performance for our Gauss-PSGD and the baseline method on \textbf{MNIST} dataset. \textbf{Top: Test accuracy} v.s. \# epoch for varying privacy budget $\epsilon\in\{0.5, 1.0, 2.0\}$. \textbf{Bottom: Test loss} v.s. \# epoch for varying privacy budget $\epsilon\in\{0.5, 1.0, 2.0\}$.}
   \label{fig:MNIST}
}

\end{figure*}

\begin{figure*}[htbp]
\centering
{
  \subfigure[Test Accuracy ($\epsilon = 0.5$)]{\includegraphics[width=0.32\textwidth]{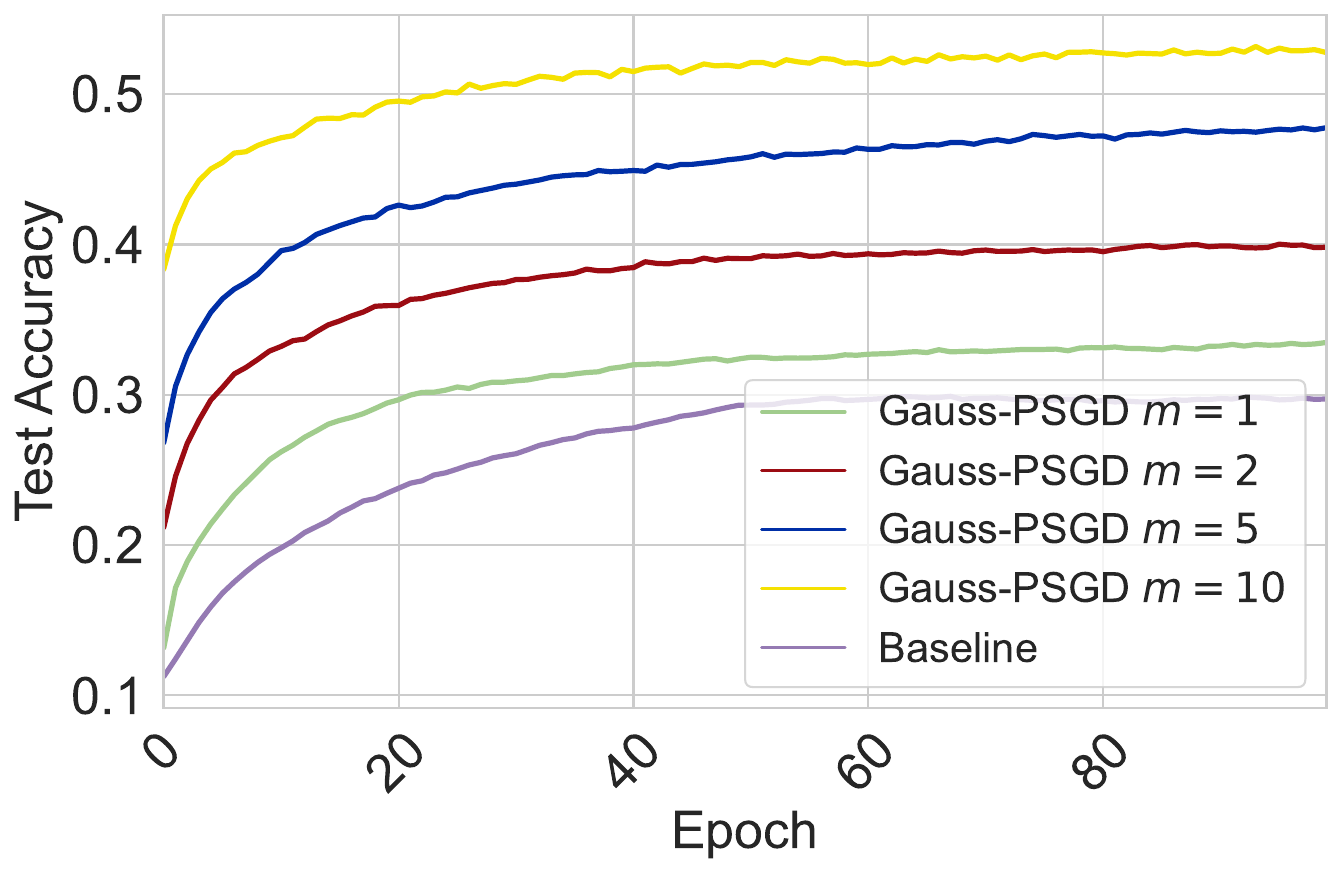}
  }\hfill
  \subfigure[Test Accuracy ($\epsilon = 1.0$)]{\includegraphics[width=0.32\textwidth]{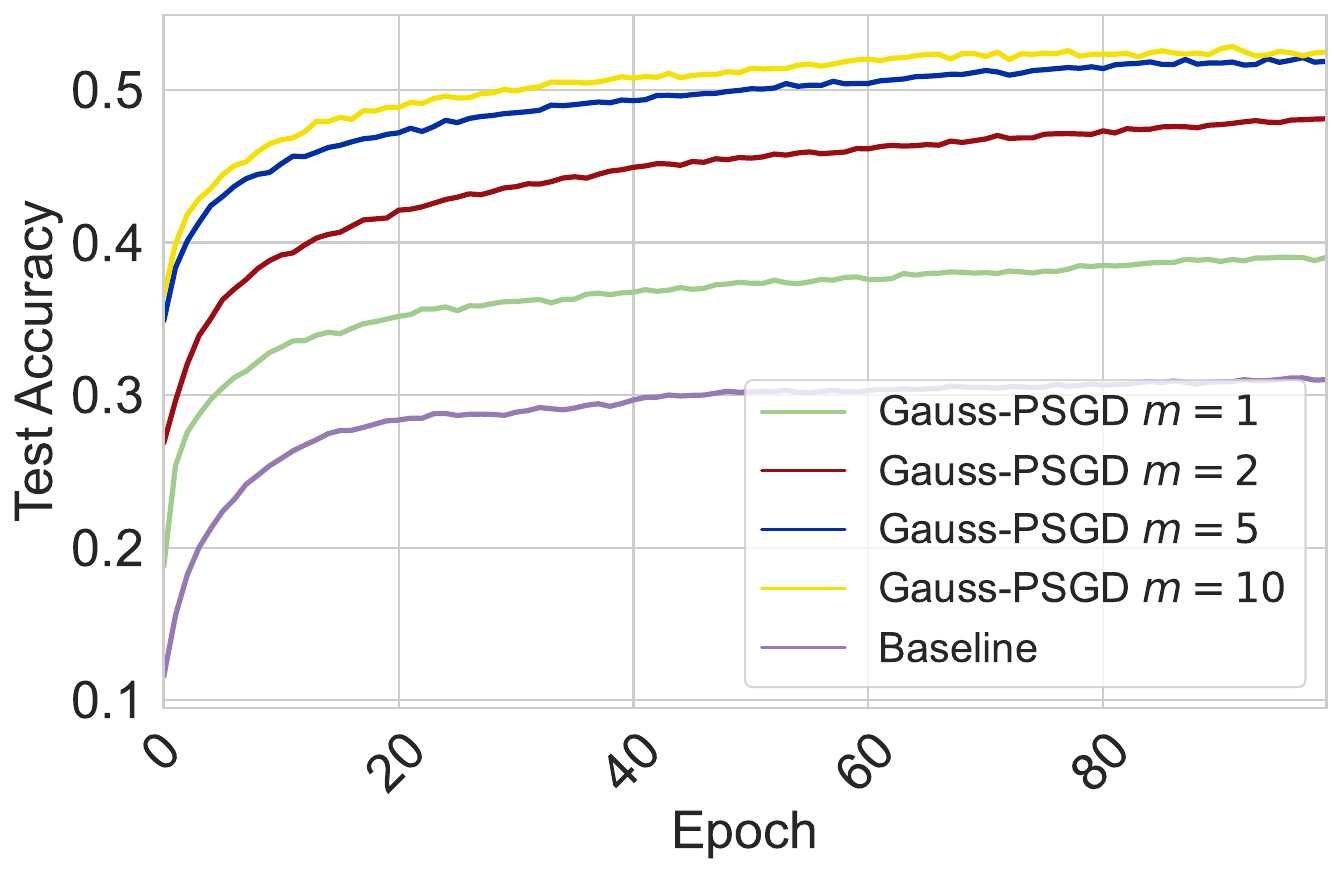}
  }\hfill
  \subfigure[Test Accuracy ($\epsilon = 2.0$)]{\includegraphics[width=0.32\textwidth]{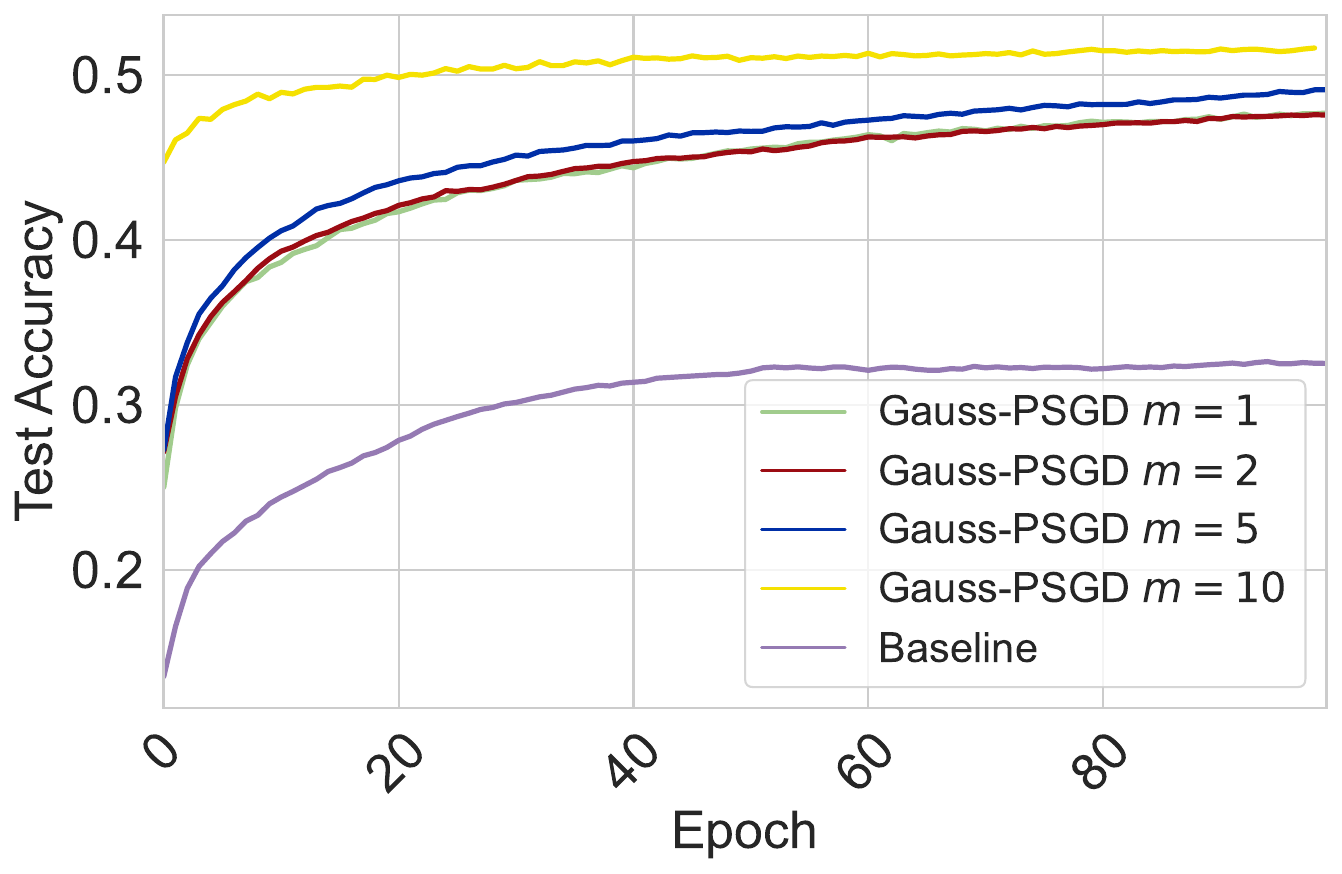}
  }\hfill
  \\
  \subfigure[Test Loss ($\epsilon = 0.5$)]{\includegraphics[width=0.32\textwidth]{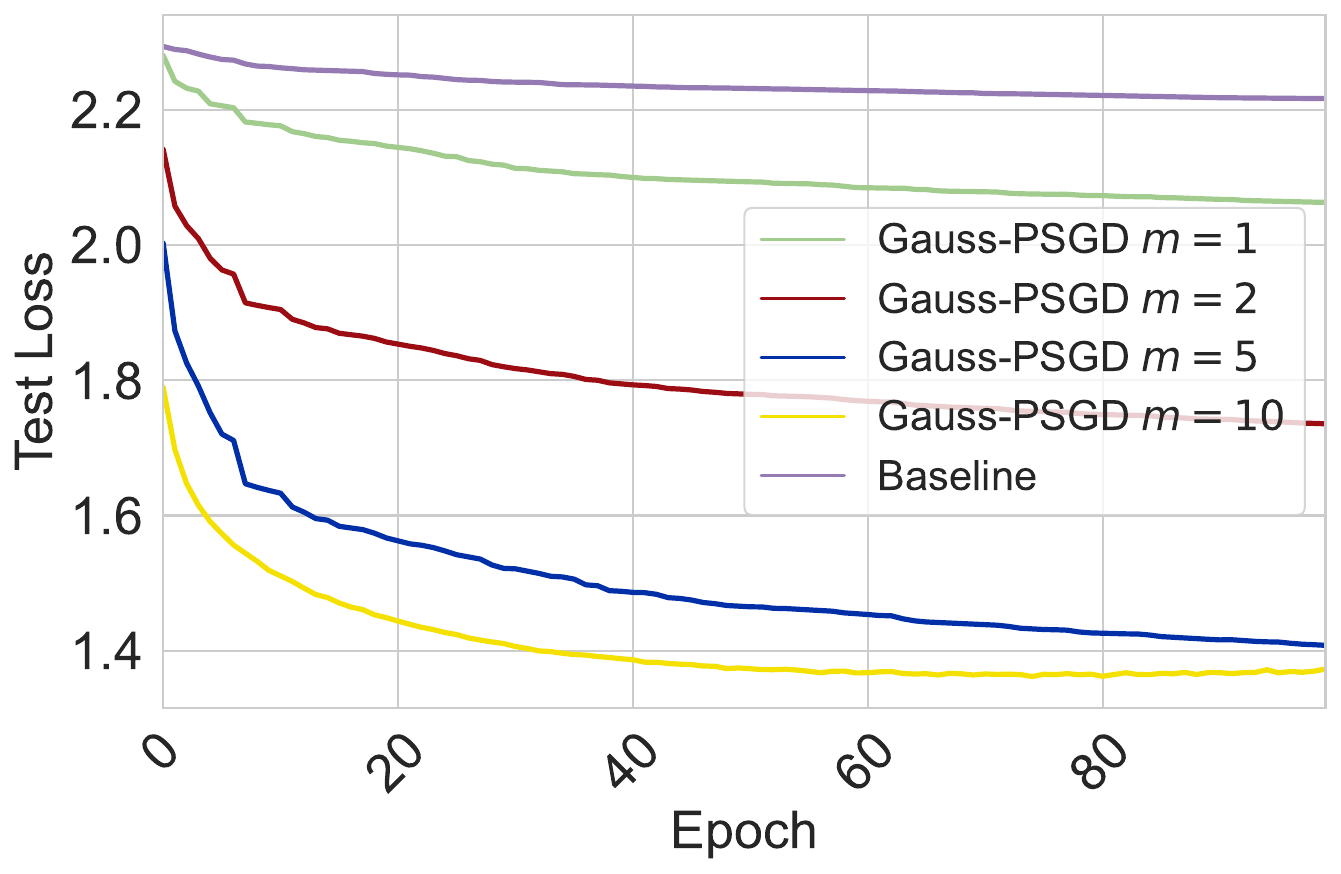}
  }\hfill
  \subfigure[Test Loss ($\epsilon = 1.0$)]{\includegraphics[width=0.32\textwidth]{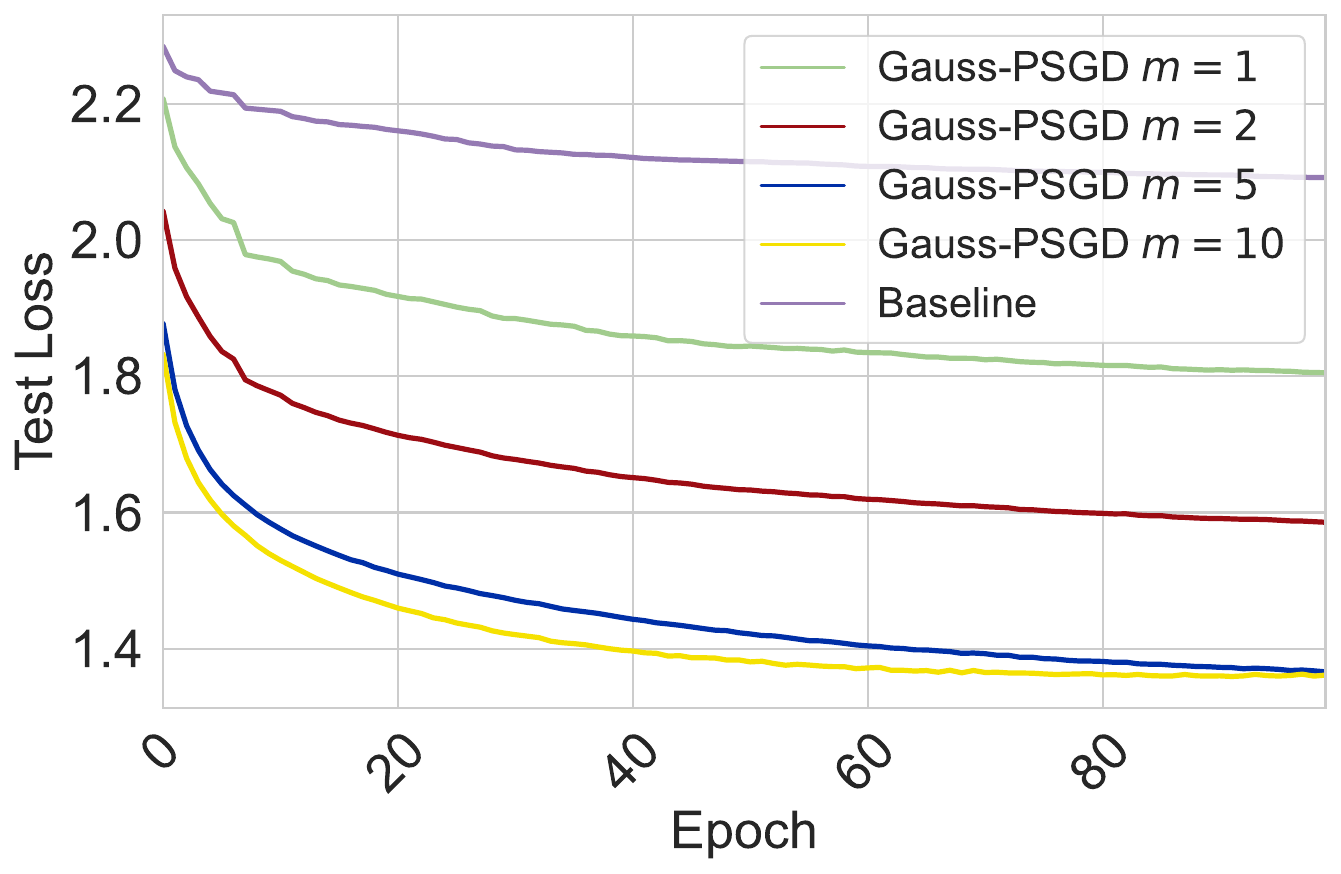}
  }\hfill
  \subfigure[Test Loss ($\epsilon = 2.0$)]{\includegraphics[width=0.32\textwidth]{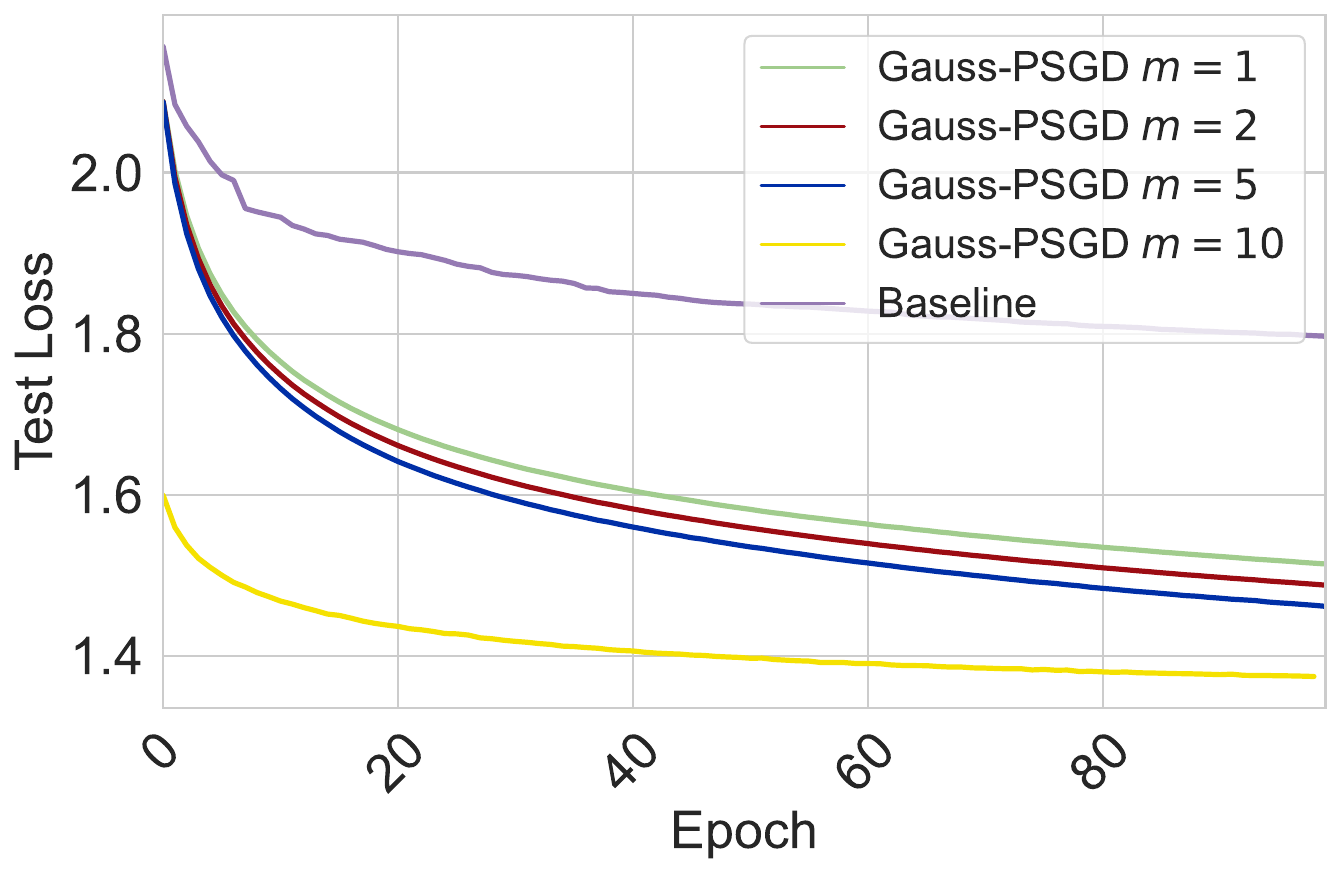}
  }\hfill
   \caption{Comparison of learning performance for our Gauss-PSGD and the baseline method on \textbf{CIFAR-10} dataset. \textbf{Top: Test accuracy} v.s. \# epoch for varying privacy budget $\epsilon\in\{0.5, 1.0, 2.0\}$. \textbf{Bottom: Test loss} v.s. \# epoch for varying privacy budget $\epsilon\in\{0.5, 1.0, 2.0\}$.}
   \label{fig:CIFAR-10}
}

\end{figure*}

\paragraph{Tasks and Datasets}
We conduct image classification tasks on two datasets: MNIST~\cite{lecun1998gradient} and CIFAR-10~\cite{krizhevsky2009learning}. For each experiment, we set the number of training samples to $n=6000$ and vary the number of clients $m$ in $\{1, 2, 5, 10\}$, where $m=1$ corresponds to the single-machine setting, while the others correspond to distributed learning scenarios. The test set consists of $10000$ samples for both datasets.

\paragraph{Models}
 
We use a fully connected neural network with one hidden layer containing 128 units and ReLU activation. The loss function is the standard cross-entropy loss. The model is initialized using Kaiming initialization~\cite{he2015delving}, with biases set to zero by default.

\paragraph{Algorithms}
We compare our proposed algorithm, which is abbreviated as \textbf{Gauss-PSGD}, against the baseline method from \cite{liu2024private}. The hyperparameters for Gauss-PSGD are set as follows:
\begin{itemize}
    \item Escape threshold $\chi=0.01$
    \item Model drift threshold $\kappa=0.1$
    \item Maximum escape steps $\Gamma = 10$
    \item Maximum repeat number of escape $Q = 3$
\end{itemize}
For all algorithms, we set the privacy parameters to $\delta=10^{-5}$ and vary $\epsilon$ in $\{0.5, 1.0, 2.0\}$, corresponding to strong, medium, and weak privacy regimes, respectively. The learning rate is set to $0.001$ for MNIST and $0.01$ for CIFAR-10.

\paragraph{Evaluations}

We evaluate the performance of the implemented algorithms using two criteria: test accuracy and test loss. Both metrics are analyzed over training epochs to assess convergence and generalization performance.

\paragraph{Results}

The experimental results for the MNIST and CIFAR-10 datasets are shown in Fig.~\ref{fig:MNIST} and Fig.\ref{fig:CIFAR-10}, respectively. In each figure, we present test accuracy (top row) and test loss (bottom row) against the number of epochs for different privacy budgets ($\epsilon=0.5,1.0,2.0$). From the experimental results, it can be seen that our proposed Gauss-PSGD consistently outperforms the baseline across all configurations. Specifically, Gauss-PSGD achieves higher test accuracy than the baseline for both datasets, with accuracy improving as $\epsilon$ increases due to weaker privacy constraints. The accuracy gap between Gauss-PSGD and the baseline widens in distributed settings ($m>1$), highlighting the collaborative synergy of distributed learning and the robustness of Gauss-PSGD in handling data heterogeneity. Gauss-PSGD exhibits lower test loss compared to the baseline across all configurations. The rapid reduction in loss during the initial epochs indicates faster convergence, which holds true for both datasets and all privacy budgets. In conclusion, the results demonstrate that Gauss-PSGD achieves superior accuracy, faster convergence, and better scalability compared to the baseline.

\section{Broader Impact Statement}\label{sec-impact}

This paper advances the field of differentially private (DP) stochastic non-convex optimization by addressing key theoretical challenges in finding second-order stationary points (SOSP). Our contributions are particularly relevant for applications requiring strong privacy guarantees, including distributed learning with heterogeneous data. These advancements have practical implications for privacy-sensitive fields such as healthcare, finance, and large language models (LLMs), where data confidentiality is paramount.

By improving the efficiency and accuracy of DP optimization techniques, our work supports the development of machine learning systems that can operate on sensitive datasets without compromising privacy. This fosters greater trust in data-driven decision-making and encourages organizations to adopt privacy-preserving practices, enabling informed and responsible use of sensitive data.

Nevertheless, it is important to acknowledge the broader limitations inherent to DP-based learning algorithms, not just those specific to our work. Privacy-preserving methods often introduce trade-offs, such as reduced model accuracy compared to their non-private counterparts, which may impact decision-making in high-stakes applications.

Despite these challenges, we believe that advancing and responsibly applying privacy-preserving optimization techniques will have a positive societal impact. By enabling secure and ethical data analysis, our work contributes to the broader goal of building trustworthy AI/ML systems.

\end{document}